\documentclass[twoside]{article}

\usepackage[accepted]{aistats2019}
\usepackage[round]{natbib}

\usepackage{tikz}
\usetikzlibrary{arrows.meta}
\usepackage{xcolor}
\usepackage{amsmath}
\usepackage{amsthm}
\usepackage{bm}
\usepackage{amssymb}
\usepackage{dsfont}
\usepackage{algorithm}
\usepackage{algorithmic}
\usepackage{enumitem}
\usepackage{appendix}

\usepackage[colorinlistoftodos, textwidth=20mm]{todonotes}
\definecolor{citrine}{rgb}{0.89, 0.82, 0.04}
\definecolor{blued}{RGB}{70,197,221}

\usepackage{mathtools}
\usepackage{graphicx,wrapfig,lipsum}
\usepackage{caption}
\usepackage{subcaption}
\usepackage{hhline}
\usepackage{float}

\newcommand{\A}{\mathcal A}
\newcommand{\F}{\mathcal F}

\newcommand{\calS}{\mathcal S}

\renewcommand{\Re}{\mathbb R}
\newcommand{\Na}{\mathbb N}

\newcommand{\reg}{\text{reg}}

\newcommand{\transp}{\mathsf{T}}

\DeclarePairedDelimiter{\abs}{\lvert}{\rvert}

\newcommand{\ratio}{{\small\textsc{RATIO}}\xspace}
\newcommand{\loss}{{\small\textsc{LOSS}}\xspace}
\newcommand{\al}{{\small\textsc{AL}}\xspace}
\newcommand{\mab}{{\small\textsc{MAB}}\xspace}
\newcommand{\mdp}{{\small\textsc{MDP}}\xspace}
\newcommand{\fw}{{\small\textsc{FW}}\xspace}
\newcommand{\fwmdp}{{\small\textsc{FW-AME}}\xspace}
\newcommand{\fmh}{{\small\textsc{FMH}}\xspace}
\newcommand{\fmhsdp}{{\small\textsc{FMH-SDP}}\xspace}

\newcommand{\wt}[1]{\widetilde{#1}}
\newcommand{\wh}[1]{\widehat{#1}}

\newcommand{\wu}[1]{\underline{#1}}

\usepackage{accents}
\DeclareMathAccent{\wtilde}{\mathord}{largesymbols}{"65}

\usepackage{xspace}

\theoremstyle{theorem}
\newtheorem{assumption}{Assumption}
\newtheorem{proposition}{Proposition}
\newtheorem{lemma}{Lemma}
\newtheorem{theorem}{Theorem}

\newlength{\minipagewidth}
\newlength{\minipagewidthx}
\begin{document}

\twocolumn[

\aistatstitle{Active Exploration in Markov Decision Processes}

\aistatsauthor{ Jean Tarbouriech \And Alessandro Lazaric }

\aistatsaddress{ Facebook AI Research \And  Facebook AI Research } ]

\begin{abstract}
	We introduce the active exploration problem in Markov decision processes (MDPs). Each state of the MDP is characterized by a random value and the learner should gather samples to estimate the mean value of each state as accurately as possible. Similarly to active exploration in multi-armed bandit (MAB), states may have different levels of noise, so that the higher the noise, the more samples are needed. As the noise level is initially unknown, we need to trade off the \textit{exploration} of the environment to estimate the noise and the \textit{exploitation} of these estimates to compute a policy maximizing the accuracy of the mean predictions. We introduce a novel learning algorithm to solve this problem showing that active exploration in MDPs may be significantly more difficult than in MAB. We also derive a heuristic procedure to mitigate the negative effect of slowly mixing policies. Finally, we validate our findings on simple numerical simulations.
\end{abstract}



\vspace{-0.1in}
\section{Introduction}
\vspace{-0.1in}


\textit{Active exploration}\footnote{We use this term in contrast to the exploration-exploitation dilemma (i.e., regret minimization) and best-arm identification.} refers to the problem of actively querying an unknown environment to gather information and perform accurate predictions about its behavior. Popular instances of active exploration are optimal design of experiments~\citep{pukelsheim2006optimal} and, more in general, active learning (\al)~\citep{hanneke2014theory}, where given a fixed budget of samples, a learner actively chooses where to query an unknown function to collect information that could maximize the accuracy of its predictions. An effective \al method should adjust to the \textit{approximation function space} to obtain samples wherever the uncertainty is high. In multi-armed bandit (\mab), the active exploration problem~\citep{antos2010active,carpentier2011upper} rather focuses on adjusting to the \textit{noise} affecting the observations, which may differ over arms. Despite their difference, in both \al and \mab, the underlying assumption is that the learner can \textit{directly} collect a sample at any arbitrary point or pull any arm with no constraint.

In this paper, we extend the \mab setting to active exploration in a Markov decision process (\mdp), where each state (an arm in the \mab setting) is characterized by a random variable that we need to estimate. Unlike \al and \mab, if the learner needs to generate an ``experiment'' at a state, it needs to move from the current state to the desired state. Consider the problem of accurately measuring the level of pollution over different sites when a fixed budget of measurements is provided and only one measuring station is available. 
The noise affecting the observations may differ over sites and we need to carefully design a policy in order to collect more samples (resp. less samples) on sites with higher (resp. lower) noise. Since the noise level may be unknown in advance, this requires alternating between the exploration of the environment to estimate the noise level and the exploitation of the estimates to optimize the collection of ``useful'' samples. 

The main contributions of this paper can be summarized as follows: \textbf{1)} we introduce the active exploration problem in MDP and provide a thorough discussion on its difference w.r.t.\ the bandit case, \textbf{2)} inspired by the bandit algorithm of~\citet{carpentier2011upper} and Frank-Wolfe UCB by~\citet{berthet2017fast}, we devise a novel learning algorithm with vanishing regret under the assumption that the MDP is ergodic and its dynamics is known in advance, \textbf{3)} we discuss how slowly mixing policies may compromise the estimation accuracy and introduce a heuristic convex problem to compute faster mixing reversible policies, \textbf{4)} we report numerical simulations on simple MDPs to validate our theoretical findings. Finally, we discuss how our assumptions (e.g., known dynamics) could be relaxed.

\textbf{Related work.} 
\citet{dance2017optimal} study active exploration in restless bandit where the value of each arm is not an i.i.d.\ random variable but has a stationary dynamics. Nonetheless, they still consider the case where any arm can be pulled at each time step. Security games, notably the patrolling problem~\cite[e.g.,][]{basilico2012patrolling}, often consider the dynamics of moving from a state to another, but the active exploration is designed to contrast an adversary ``attacking'' a state~\cite[e.g.,][]{balcan2015commitment}. \citet{rolf2018a-successive-elimination} consider the problem of navigating a robot in an environment with background emissions to identify the $k$ strongest emitters. While the performance depends on how the robot traverses the environment, the authors only consider a fixed sensing path. \citet{auer2011models} study the autonomous exploration problem, where the objective is to discover the set of states that are reachable (following a shortest path policy) within a given number of steps. Intrinsically motivated reinforcement learning~\citep{chentanez2005intrinsically} often tackles the problem of ``discovering'' how the environment behaves (e.g., its dynamics) by introducing an ``internal'' reward signal. \citet{maxent2018} recently focus on the instrinsically-defined objective of learning a (possibly non-stationary) policy that induces a state distribution that is as uniform as possible (i.e., with maximal entropy). This problem is related to our setting in the special case of equal state variances. We believe such line of work is insightful as it may help to understand how to encourage an agent to find policies which can manipulate its environment in the absence of any extrinsic scalar reward signal.


\vspace{-0.1in}
\section{Preliminaries}\label{sec:preliminaries}
\vspace{-0.1in}

\textbf{Active exploration in MDPs.} A Markov decision process (MDP)~\citep{puterman1994markov} is a tuple $M = (\mathcal{S}, \mathcal{A}, p, \nu, \overline{s})$, where $\mathcal{S}$ is a set of $S$ states, $\mathcal{A}$ is a set of $A$ actions, and for any $s,a\in \calS\times \A$, $p(s'|s,a)$ is the transition distribution over the next state $s'\in\calS$. We also define the adjacency matrix $Q\in\Re^{S\times S}$, such that $Q(s,s') = 1$ for any $s,s'\in\calS$ where there exists an action $a\in\A$ with $p(s'|s,a)>0$, and $Q(s,s')=0$ otherwise. Instead of a reward function, 
$\nu(s)$ is an observation distribution supported in $[0,R]$ with mean $\mu(s)$ and variance $\sigma^2(s)$, characterizing the random event that we want to accurately estimate on each state. 
Finally, $\overline{s}$ is the starting state. 
The stochastic process works as follows. At step $t=1$ the environment is initialized at $s_1 = \overline{s}$, an agent takes an action $a_1$, which triggers a transition to the next state $s_2 \sim p(\cdot|s_1,a_1)$ and an observation $x_2\sim\nu(s_2)$, and so on. 
We denote by $\F_t = \{s_1,a_1,s_2,x_2,a_2, \ldots, s_t,x_t\}$ the history up to $t$. A randomized history-dependent (resp. stationary) policy $\pi$ at time $t$ is denoted by $\pi_t : \F_t \rightarrow \Delta(\A)$ (resp. $\pi : \calS \rightarrow \Delta(\A)$) and it maps the history (resp. the current state) to a distribution over actions. We denote the set of history-dependent (resp. stationary) policies by $\Pi^{\textrm{HR}}$ (resp. $\Pi^{\textrm{SR}}$). For any policy $\pi$, we denote by $T_{\pi,n}(s) = \sum_{t=2}^{n} \mathbb{I}\{s_t = s\}$ the number of observations collected in state $s$ when starting from $s_1 = \overline{s}$ and following policy $\pi$ for $n$ steps.\footnote{The counter starts at $2$ as observations are received upon arrival on a state (i.e., no observation at $s_1=\overline{s}$).} At the beginning of step $t$, for any state $s$ such that $T_{\pi,t}(s) > 0$, the empirical estimates of the mean and variance are computed as
\begin{equation}\label{eq:emp.variables}
\begin{aligned}
\wh\mu_{\pi,t}(s) &\!=\! \frac{1}{T_{\pi,t}(s)} \sum_{\tau=2}^t x_{\tau} \mathbb{I}\{s_{\tau}\!=\!s\}\\
\wh\sigma^2_{\pi,t}(s) &\!=\! \frac{1}{T_{\pi,t}(s)} \sum_{\tau=2}^t x_{\tau}^2 \mathbb{I}\{s_{\tau}\!=\!s\} - \wh\mu_{\pi,t}^2(s)
\end{aligned}.
\end{equation}
In order to avoid dealing with subtle limit cases and simplify the definition of the estimation problem, we introduce the following assumption.

\begin{assumption}\label{asm:counter}
For any state $s\in\calS$ and  policy $\pi$, $T_{\pi,1}(s) = 1$ and $T_{\pi,n}(s) = 1+\sum_{t=2}^{n} \mathbb{I}\{s_t = s\}$.
\end{assumption}

We basically assume that at $t=1$ one sample is available and used in estimating $\mu(s)$ and $\sigma^2(s)$ at each state (see App.~\ref{app:relaxing:asm:counter} for further discussion). For any policy $\pi$ and any budget  $n\in\Na$, we define the estimation problem as the minimization of the normalized mean-squared estimation error
\begin{align*}
\min_{\pi \in\Pi^{\textrm{HR}} } \mathcal{L}_{n}(\pi):= \frac{n}{S} \sum_{s\in\calS} \mathbb{E}_{\pi,\nu} \Big[ \big( \wh\mu_{\pi,n}(s) - \mu(s)\big)^2\Big],
\end{align*}
where $\mathbb{E}_{\pi,\nu}$ is the expectation w.r.t.\ the trajectories generated by $\pi$ and the observations from $\nu$. When the dynamics $p$ and the variances $\sigma^2(s)$ are known, we restrict our attention to stationary polices $\pi \in\Pi^{\textrm{SR}}$ and exploiting the independence between transitions and observations, and Asm.~\ref{asm:counter}, we obtain 
\begin{align}\label{eq:opt.function.simplified}
\mathcal{L}_{n}(\pi) &= \frac{n}{S} \sum_{s\in\calS} \mathbb{E}_{\pi} \bigg[ \mathbb{E}_{\nu} \Big[\big( \wh\mu_{\pi,n}(s) - \mu(s)\big)^2 \Big|T_{\pi,n} \Big]\bigg] \nonumber \\
&= \frac{1}{S} \sum_{s\in\calS} \sigma^2(s) \mathbb{E}_{\pi} \bigg[ \frac{n}{T_{\pi,n}(s)}\bigg].
\end{align}
In the case of deterministic and fully-connected MDPs, the problem smoothly reduces to the active bandit formulation of~\citet{antos2010active}.


\textbf{Technical tools.} 
For any stationary policy $\pi\in\Pi^{\textrm{SR}}$, we denote by $P_\pi$ the kernel of the Markov chain induced by $\pi$ in the MDP, i.e., $P_\pi(s'|s) = \sum_{a\in\A} p(s'|s,a) \pi(a|s)$. 
If the Markov chain $P_\pi$ is ergodic (i.e., all states are aperiodic and recurrent), it admits a unique stationary distribution over states $\eta_\pi$, such that $\eta_\pi(s) = \sum_{s'} P_\pi(s|s')\eta_\pi(s')$. A Markov chain $P_\pi$ is reversible if the detailed balance condition 
%
$\eta_\pi(s)P_\pi(s'|s) = \eta_\pi(s')P_\pi(s|s')$ 
%
is satisfied for all $s,s'\in\calS$.
Let $\{\xi_\pi(s)\}$ be the eigenvalues of $P_\pi$, we define the second-largest eigenvalue modulus (SLEM) and the spectral gap as
\begin{align}\label{eq:slem.spectral.gap}
\xi_{\pi,\max} := \max_{s:\xi_\pi(s) \neq 1} |\xi_\pi(s)|;\;\;\; \gamma_\pi := 1 - \xi_{\pi,\max}.
\end{align}
The SLEM can be written as the spectral norm (i.e., the maximum singular value) of an affine matrix in $P_{\pi}$. Let $D_\eta$ be the diagonal matrix with the elements of $\eta$, then~\citep{boyd2004fastest}
\begin{align} \label{SLEM_as_norm}
\xi_{\pi,\max} = \| D_{\eta_{\pi}}^{1/2} P_{\pi} D_{\eta_{\pi}}^{-1/2} - \sqrt{\eta_{\pi}} \sqrt{\eta_{\pi}}^T \|_2.
\end{align}
For ergodic chains, $\xi_{\pi,\max} < 1$. The spectral gap is tightly related to the mixing time of the chain and it characterizes how fast the frequency of visits converges to the stationary distribution (e.g., \citet[Thm.~3,][]{Hsu2015mixing},~\citet[Thm.~3.8,][]{Paulin}).

\begin{proposition}\label{prop.concentration.mixing}
Let $\pi\in\Pi^{\textrm{SR}}$ be a stationary policy inducing an ergodic and reversible chain $P_{\pi}$ with spectral gap $\gamma_\pi$ and stationary distribution $\eta_\pi$. Let $\eta_{\pi,\min} = \min_{s \in \calS} \eta_\pi(s)$. For any budget $n > 0$ and state $s\in\calS$,
\begin{align*}
\Big| \frac{\mathbb{E}\big[T_{\pi,n}(s)\big]}{n} - \eta_\pi(s) \Big| \leq \dfrac{1}{2 \sqrt{\eta_{\pi,\min}}} \frac{1}{\gamma_\pi n},
\end{align*}
and for any $\delta\in(0,1)$, with probability $1-\delta$,
\begin{align*}
\Big| \frac{T_{\pi,n}(s)}{n} \!-\! \eta_\pi(s) \Big| \leq \epsilon_\pi(s,n,\delta) := O \Big(\sqrt{\dfrac{\ln(\frac{1}{\delta}\sqrt{\frac{2}{\eta_{\pi,\min}}})}{\gamma_\pi n}}\Big).
\end{align*}
\end{proposition}

The exact formulation of $\epsilon_\pi(s,n,\delta)$ is reported in App.~\ref{app:proofs} (see proof of~Lem.~\ref{lem:asmyptoric.performance}). It is interesting to notice that the convergence in expectation is faster than in high-probability ($O(n^{-1})$ vs $O(n^{-1/2})$), but in both cases the spectral gap may significantly affect the convergence (e.g., for slowly mixing chains $\gamma_{\pi} \approx 0$).

Finally, we recall a concentration inequality for variance estimation (see~\citet{antos2010active}).

\begin{proposition}\label{lem:bound_variance_estimates}
For any $\delta \in (0,1)$ and time $t$, with probability at least $1-\delta$ 
\begin{align*}
\big|\widehat{\sigma}_t^2(s) - \sigma^2(s)\big| \leq \alpha(t,s,\delta) := 5 R^2 \sqrt{\dfrac{\log(\frac{4St}{\delta})}{T_t(s)}}.
\end{align*}
\end{proposition}

\vspace{-0.1in}
\section{The Asymptotic Case}\label{sec:asymptotic}
\vspace{-0.1in}

In deterministic fully-connected MDPs,  problem~\eqref{eq:opt.function.simplified} reduces to the bandit setting and it also inherits its NP-hard complexity, as it may require enumerating all possible values of $\{T_n(s)\}_s$~\citep[see e.g.,][]{welch1982algorithmic}. 
In our case, this difficulty is further increased by the fact that observations can only be collected through the ``constraint'' of the MDP dynamics. In this section we introduce an asymptotic version of the estimation problem and a learning algorithm with vanishing regret w.r.t.\ the optimal asymptotic stationary policy.

\vspace{-0.1in}
\subsection{An Asymptotic Formulation}\label{ssec:asmyptotic.formulation}
\vspace{-0.1in}

A standard approach in experimental optimal design~\citep{pukelsheim2006optimal} and \mab~\citep{antos2010active,carpentier2011upper} is to replace  problem~\eqref{eq:opt.function.simplified} by its continuous relaxation, where the empirical frequency $T_n(s)/n$ is replaced by a distribution over states. In our case $T_n(s)$ cannot be directly selected so we rather consider an asymptotic formulation for $n\rightarrow \infty$.\footnote{In the bandit case, the continuous relaxation is equivalent to the asymptotic formulation.} We first introduce the following assumption on the MDP.

\begin{assumption}\label{asm:ergodic}
	For any stationary policy $\pi\in\Pi^{\textrm{SR}}$, the corresponding Markov chain $P_\pi$ is ergodic and we denote by $\eta_{\min} = \inf_{\pi\in\Pi^{\textrm{SR}}} \min_{s\in\calS}\eta_\pi(s)$ the smallest stationary probability across policies.
\end{assumption}

Asm.~\ref{asm:counter} and~\ref{asm:ergodic}, together with the continuity of the inverse function $x \mapsto 1/x$ on $[1/n, 1]$, guarantee that for any policy $\pi$, $\frac{n}{T_{\pi,n}(s)}$ converges almost-surely to $\frac{1}{\eta_\pi(s)}$ when $n \rightarrow +\infty$ (see Prop.~\ref{prop.concentration.mixing}). As a result, we replace problem~\eqref{eq:opt.function.simplified} with
\begin{equation}\label{eq:opt.function.asymptotic}
\begin{aligned}
&\min_{\pi \in\Pi^{\textrm{SR}}, \eta \in \Delta(\calS) } \mathcal{L}(\pi,\eta) := \frac{1}{S} \sum_{s\in\calS}  \frac{\sigma^2(s)}{\eta(s)} \\
&\text{s.t.} \; \forall s \in\calS,\; \eta(s) = \sum_{s',a} \pi(a|s')p(s|s',a)\eta(s') 
\end{aligned},
\end{equation}
where $\eta$ is constrained to be the stationary distribution associated with $\pi$ (i.e., $\eta = \eta_\pi$). While both $\pi$ and $\eta$ belong to a convex set and $\mathcal{L}(\pi, \eta)$ is convex in $\eta$, the overall problem is not convex because of the constraint. Yet, we can apply the same reparameterization used in the dual formulation of reward-based MDP~\citep[Sect.~8,][]{puterman1994markov} and introduce the state-action stationary distribution $\lambda_\pi \in \Delta(\calS\times \A)$ of a policy $\pi$. Let 
\begin{align}\label{eq:state.action.stationary}
&\Lambda =  \Big\{\lambda \in \Delta(\calS\times \A): ~ \forall s\in\calS, \nonumber \\
&\;\; \displaystyle\sum_{b \in \mathcal{A}} \lambda(s,b) = \!\!\sum_{s' \in \mathcal{S}, a \in \mathcal{A}}\!\! p(s|s',a)\lambda(s',a)\Big\}
\end{align}
be the set of state-action stationary distributions, we define the optimization problem
\begin{equation}\label{eq:opt.function.asymptotic.lambda}
\begin{aligned}
&\min_{\lambda \in \Delta(\calS\times \A) } \mathcal{L}(\lambda) := \frac{1}{S} \sum_{s\in\calS}  \frac{\sigma^2(s)}{\sum_{a\in\A}\lambda(s,a)} \\
&\text{subject to} \quad \lambda \in \Lambda
\end{aligned}.
\end{equation}
%
We can characterize this problem as follows.

\begin{proposition}\label{prop:asymptotic.relaxation}
The function $\mathcal{L}(\lambda)$ is convex on the convex set $\Lambda$. Let $\lambda^{\star}$ be the solution of~\eqref{eq:opt.function.asymptotic.lambda}, then the policy
\begin{align}\label{eq:asymptotic.opt.policy}
\pi_{\lambda^{\star}}(a | s) = \dfrac{\lambda^{\star}(s,a)}{\sum_{b \in \mathcal{A}} \lambda^{\star}(s,b)}, ~ \forall s \in \mathcal{S}, a \in \mathcal{A}
\end{align}
belongs to $\Pi^{\textrm{SR}}$ and solves problem~\eqref{eq:opt.function.asymptotic}. Furthermore for any $\wu{\eta}>0$, $\mathcal{L}(\lambda)$ is $C_{\wu\eta}$-smooth on the restricted set $\Lambda_{\wu{\eta}} = \{\lambda \in \Lambda:  \sum_{a \in \mathcal{A}} \lambda(s,a) \geq 2 \wu{\eta},~\forall s\in\calS\}$ with parameter $C_{\wu\eta} \leq A \sum_{s \in \mathcal{S}} \sigma^2(s)/(2\underline{\eta})^3$. 
\end{proposition}

As a result, whenever the dynamics of the MDP and the variances $\sigma^2(s)$ are known, problem~\eqref{eq:opt.function.asymptotic.lambda} can be efficiently solved using any optimization algorithm for convex and smooth functions (e.g., projected gradient descent or Frank-Wolfe~\citep{Jaggi}). 
Leveraging Prop.~\ref{prop.concentration.mixing} we can also characterize the difference between the solutions of the asymptotic problem~\eqref{eq:opt.function.asymptotic} and the finite-budget one~\eqref{eq:opt.function.simplified}. For the sake of simplicity and at the cost of generality (see App.~\ref{app:relaxing:asm:reversible}), we introduce an additional assumption.

\begin{assumption}\label{asm:reversible}
	For any stationary policy $\pi\in\Pi^{\textrm{SR}}$, the corresponding Markov chain $P_\pi$ is reversible and we denote by $\gamma_{\min} = \min_{\pi\in\Pi^{\textrm{SR}}} \gamma_\pi$ the smallest spectral gap across all policies.
\end{assumption}

\begin{lemma}\label{lem:asmyptoric.performance}
Let $\delta = SA^S/ n^2$, if $n$ is big enough such that for any $s\in\calS$ and any stationary policy $\pi\in\Pi^{\textrm{SR}}$, $\epsilon_\pi(s,n,\delta) \leq \eta_\pi(s)/2$, then we have
\begin{align}\label{eq:concentration.loss}
\big| \mathcal{L}_n(\pi) - \mathcal{L}(\pi, \eta_\pi) \big| \leq  \ell_n(\pi),
\end{align}
where
\begin{align*}
\ell_n(\pi) := \frac{1}{S \sqrt{\eta_{\min}}n\gamma_\pi}\sum_{s\in\calS} \frac{\sigma^2(s)}{\eta_\pi^2(s)}\Big( 1 + 2 \frac{\epsilon_\pi(s,n,\delta)}{\eta_\pi(s)}\Big),
\end{align*}
which gives the performance loss
\begin{align}\label{eq:performance.loss}
\mathcal{L}_n(\pi_{\lambda^\star}) - \mathcal{L}_n(\pi^\star_n) &\leq \ell_n(\pi_{\lambda^\star}) + \ell_n(\pi^\star_n),
\end{align}
where $\pi^\star_n$ is the solution to problem~\eqref{eq:opt.function.simplified} and $\pi_{\lambda^\star}$ is defined in~\eqref{eq:asymptotic.opt.policy}.
\end{lemma}

It is interesting to compare the result above to the bandit case. For $n \geq 4/(S\eta_{\min}^2)$ the performance loss of the continuous relaxation in bandit is bounded as $8 \sigma^2_{\max} / (\eta_{\min}^3 n^2)$ (see Prop.~\ref{prop:bandit.perf.loss} in App.~\ref{app:proofs}). While the condition on $n$ in Lem.~\ref{lem:asmyptoric.performance} is similar (i.e., from the definition of $\epsilon_\pi(s,n,\delta)$, we need $n > \wt{\Omega}(1/\eta_{\min}^2)$), the performance loss differs over two main elements: \textit{(i)} the rate of convergence in $n$, \textit{(ii)} the presence of the spectral gap $\gamma_\pi$. In \mab, the ``fast'' convergence rate is obtained by exploiting the smoothness of the function $\mathcal{L}$, which characterizes the performance of both discrete and continuous allocations. On the other hand, in the MDP case, while $\mathcal{L}$ is indeed smooth on the restricted simplex, $\mathcal{L}_n$ is a more complicated function of $\pi$, which does not allow the same proof technique to be directly applied. Furthermore, the spectral gap directly influences the difference between the finite-time and asymptotic behavior of a policy $\pi$. This extra ``cost'' is not present in \mab, where any allocation over states can be directly ``executed'' without waiting for the policy to mix.


\vspace{-0.1in}
\subsection{Learning Algorithm}\label{ssec:learning}
\vspace{-0.1in}

\begin{algorithm}[t]
	\caption{\fwmdp: the \textit{Frank-Wolfe for Active MDP Exploration} algorithm}
	\label{alg:FW-AME}
	\begin{algorithmic} 
		\STATE \textbf{Input:} $\wt{\lambda}_1 = 1/SA$, $\wu \eta$\\
		\FOR{$k=1, 2, ..., K-1$}
		\STATE $\wh{\psi}^{+}_{k+1} = \textrm{argmin}_{\lambda \in \Lambda_{\wu{\eta}}} \langle \nabla \wh{\mathcal{L}}_{t_k-1}^{+}(\wt{\lambda}_k), \lambda \rangle$
		\STATE $\wh{\pi}^{+}_{k+1}(a|s) = \dfrac{\wh{\psi}^{+}_{k+1}(s,a)}{\sum_{b \in \cal{A}}\wh{\psi}^{+}_{k+1}(s,b)}$
		\STATE Execute  $\wh{\pi}^{+}_{k+1}$ for $\tau_{k}$ steps 
		\STATE Update the state-action frequency $\wt{\lambda}_{k+1}$ 
		\ENDFOR
	\end{algorithmic}
\end{algorithm}

We introduce a learning algorithm to incrementally solve the active exploration problem
in the setting where the state variances $\sigma^2(s)$ are unknown. We rely on the following assumption.

\begin{assumption}\label{asm:known.mdp}
The MDP model $p$ is known.
\end{assumption}

In App.~\ref{app:relaxing:asm:known.mdp} we sketch a way to relax Asm.~\ref{asm:known.mdp} by following an optimistic approach similar to UCRL~\citep{jaksch2010near} in order to incorporate the uncertainty on the MDP dynamics, and we conjecture that the regret guarantees of the algorithm would remain unchanged.

Let $\wu{\eta} < 1/(2S)$ be a positive constant. Since $\mathcal{L}(\lambda)$ is smooth in $\Lambda_{\wu{\eta}}$ (Prop.~\ref{prop:asymptotic.relaxation}), it can be optimized using the Frank-Wolfe (\fw) algorithm~\citep{Jaggi}, which constructs a sequence of linear optimization problems whose solutions are used to incrementally update the candidate solution to problem \eqref{eq:opt.function.asymptotic.lambda}. In \mab, 
\citet{berthet2017fast} showed that \fw can be fed with optimistic estimates of the gradient to obtain a bandit algorithm with small regret. The resulting algorithm (Frank-Wolfe-UCB) actually reduces to the algorithm of~\citet{carpentier2011upper} when the function to optimize is the mean estimation error. The mapping from \fw to a bandit algorithm relies on the fact that the solution to the linear problem at each iteration of \fw corresponds to selecting one single arm. 
Unfortunately, in the MDP case, \fw returns a state-action stationary distribution, which cannot be directly ``executed''. We then need to adapt the bandit-\fw idea to \textit{track} the (optimistic) \fw solutions.

In Alg.~\ref{alg:FW-AME} we illustrate \fwmdp (\fw for Active MDP Exploration) which proceeds through episodes and is evaluated according to the frequency of visits of each state, i.e., $\wt\lambda_k(s) = T_{t_k-1}(s)/(t_k-1)$. At the beginning of episode $k$, \fwmdp solves an MDP with reward related to the current estimation error, so that the corresponding optimal policy tends to explore states where the current estimate of the mean $\mu(s)$ is not accurate enough. More formally, \fwmdp solves a linear problem to compute the state-action stationary distribution $\wh\psi^+_{k+1}$ minimizing the expected ``optimistic'' gradient evaluated at the current solution obtained using the confidence intervals of Prop.~\ref{lem:bound_variance_estimates}, i.e.,
\begin{align*}
\nabla \wh{\mathcal{L}}_{t_k-1}^{+}(\lambda)(s,a) = - \dfrac{\wh{\sigma}_{t_k-1}^2(s) + \alpha(t_k-1,s,\delta)}{(\sum_b \lambda(s,b))^2}.
\end{align*}
This choice favors exploration towards states whose loss is possibly high (i.e., large gradient) and poorly estimated (large confidence intervals). This step effectively corresponds to solving an MDP with a reward equal to $\nabla \wh{\mathcal{L}}_{t_k-1}^{+}$. 
Then the policy $\wh{\pi}^{+}_{k+1}$ associated to $\wh\psi^+_{k+1}$ is executed for $\tau_k$ steps and the solution $\wt\lambda_{k}$ is updated accordingly. Let $\nu_{k+1}(s,a)$ be the number of times action $a$ is taken at state $s$ during episode $k$. We can write the update rule for the candidate solution as
\begin{align*}
\wt{\lambda}_{k+1} &= \dfrac{\tau_{k}}{t_{k+1}-1} \wt{\psi}_{k+1} + \dfrac{t_k-1}{t_{k+1}-1} \wt{\lambda}_k \\ 
&= \beta_k \wt{\psi}_{k+1} + (1-\beta_k)\wt{\lambda}_k,
\end{align*}
where $\wt{\psi}_{k+1}(s,a) = \nu_{k+1}(s,a) / \tau_k$ is the frequency of visits within episode $k$ and $\beta_k = \tau_{k}/(t_{k+1}-1)$ is the weight (or learning rate) used in updating the solution. While we conjecture that a similar approach could be paired with other optimization algorithms (e.g., projected gradient descent), by building on \fw we obtain a projection-free algorithm, where at each episodes we only need to solve a specific instance of an MDP. 
%
In App.~\ref{app:proof_regret} we derive the following regret guarantee. 

\begin{theorem}\label{thm:regret_FWAME}
Let episode lengths satisfy $t_k = \tau_1 (k-1)^3 + 1$ where $\tau_1$ is the length of the first episode, i.e.,
\begin{align*}
    \tau_k = \tau_1(3k^2-3k+1) \quad \textrm{and} \quad \beta_k = \dfrac{3k^2 - 3k + 1}{k^3}.
\end{align*}
Under Asm.~\ref{asm:counter},~\ref{asm:ergodic},~\ref{asm:reversible},~\ref{asm:known.mdp}, \fwmdp satisfies with high probability\footnote{See App.~\ref{core_proof} for a more explicit bound. See App.~\ref{app:relaxing_asm} for a discussion on the relaxation of Asm.~\ref{asm:counter},~\ref{asm:reversible} and \ref{asm:known.mdp}.}
\begin{align*}
    \mathcal{L}(\wt\lambda_K) - \mathcal{L}(\lambda^{\star}) = \wt O\big(t_K^{-1/3}\big).
\end{align*}
\end{theorem}

\textbf{Sketch of the proof.} The proof combines the \fw analysis, the contribution of the estimated optimistic gradient, and the gap between the target distribution $\wh{\psi}^{+}_{k+1}$ and the empirical frequency $\wt{\psi}_{k+1}$. Let $\rho_{k+1} := \mathcal{L}(\wt{\lambda}_{k+1}) - \mathcal{L}(\lambda^{\star})$ be the regret at the end of episode $k$. 
%
Introducing $\psi^{\star}_{k+1} := \textrm{argmin}_{\lambda \in \Lambda_{\underline{\eta}}} \langle \nabla \mathcal{L}(\wt{\lambda}_k), \lambda \rangle$ and exploiting the convexity and $C_{\underline{\eta}}$-smoothness of $\mathcal{L}$, it is possible to obtain the ``recursive'' inequality
\begin{align*}
\rho_{k+1} &\leq (1-\beta_k) \rho_k + C_{\underline{\eta}} \beta_k^2 + \beta_k \epsilon_{k+1} + \beta_k \Delta_{k+1},
\end{align*}
where $\epsilon_{k+1} := \langle \nabla \mathcal{L}(\wt{\lambda}_k), \wh{\psi}^{+}_{k+1} - \psi^{\star}_{k+1} \rangle$ and $\Delta_{k+1} := \langle \nabla \mathcal{L}(\wt{\lambda}_k), \wt{\psi}_{k+1} - \wh{\psi}^{+}_{k+1} \rangle$. The term $\epsilon_{k+1}$ is an \textit{optimization error} and it can be effectively bounded exploiting the fact that $\wh{\psi}^{+}_{k+1}$ is the result of an optimistic optimization. On the other hand, the term $\Delta_{k+1}$ is a \textit{tracking error} and it can be only bounded using Prop.~\ref{prop.concentration.mixing} as $1/\sqrt{\tau_k}$. Solving the recursion for the specific choice of $t_k$ in the theorem provides the final bound.


\textbf{Remark (rate).} The most striking difference between this bound and the result of~\citet{carpentier2011upper} and~\citet{antos2010active} in \mab is the worse rate of convergence, $O(t^{-1/3})$ vs $O(t^{-1/2})$. This gap is the result of trading off  the ``optimization'' convergence speed of \fw and the tracking performance obtained by executing $\wh{\pi}^{+}_{k+1}$. 
\citet{berthet2017fast} show that in \mab, the learning rate $\beta_k$ is set to $1/t$ (as in standard \fw) to achieve a $O(t^{-1/2})$ convergence rate. In our case, we can obtain such learning rate by setting episodes of constant length $\tau_k = \tau$. Unfortunately, this scheme would suffer a constant regret. In fact, while a \fw instance where the solution $\wt\lambda_k$ is updated directly using $\wh{\psi}^{+}_{k+1}$ would indeed converge faster with such episode scheme, our algorithm cannot ``play'' the distribution $\wh{\psi}^{+}_{k+1}$ but needs to execute the corresponding policy $\wh{\pi}^{+}_{k+1}$, which gathers samples with frequency $\wt{\psi}_{k+1}$, then used to update $\wt\lambda_k$. The gap between $\wh{\psi}^{+}_{k+1}$ and $\wt{\psi}_{k+1}$ reduces the efficiency of the optimization step by introducing an additive error of order $O(1/\sqrt{\tau})$ (see Prop.~\ref{prop.concentration.mixing}), which is constant for fixed-sized episodes. As a result, the episode length is optimized to trade off between the optimization speed and tracking effectiveness. Whether this is an intrinsic issue of the active exploration in MDP or better algorithms can be devised is an open question.

\textbf{Remark (problem-dependent constants).} Investigating the proof reveals a number of other dependencies on the algorithm's and problem's parameters. First, the regret bound depends on the inverse of the parameter $\wu \eta$ used in \fwmdp to guarantee the smoothness of the function. While this may suggest to take $\wu \eta$ as large as possible, this may over-constrain the optimization problem (i.e., the set $\Lambda_{\wu\eta}$ becomes artificially too small). If $\lambda^\star$ is the solution on the ``unconstrained'' $\Lambda$, then $2\wu\eta$ should be set exactly at $\min_s \sum_a \lambda^\star(s,a)$. Furthermore, the gap between $\wh{\psi}^{+}_{k+1}$ and $\wt{\psi}_{k+1}$ is bounded using Prop.~\ref{prop.concentration.mixing}. Since the policy executed at each step is random (it depends on the samples observed at previous episodes), we need to take the worst case w.r.t. all possible stationary policies. Thus the regret presents an inverse dependency on $\gamma_{\min}$, which could be very small. Finally, the bound has a direct dependency on the number of states.



\vspace{-0.1in}
\section{The Mixing Issue}\label{sec:relaxation}
\vspace{-0.1in}

\begin{figure}[t]
	\centering
	\begin{tikzpicture}[thick,scale=0.9]
	\begin{scope}[every node/.style={circle,thick,draw}]
	\node (1) at (0,0) {$\sigma_1^2 = 1$};
	\node (2) at (2,0) {$\sigma_2^2 \ll 1$};
	\node (3) at (4,0) {$\sigma_3^2 = 1$};
	\end{scope}
	\begin{scope}[>={Stealth[black]},
	every node/.style={fill=white,circle},
	every edge/.style={draw=gray, thick},
	every loop/.style={draw=gray, thick, min distance=5mm,looseness=5}]
	\path[]
	(1) [->,thick] edge[in=200,out=150, loop left,thick] node[left] {$a_1$} (1)
	(3) [->,thick] edge[in=200,out=150, loop right,thick] node[right] {$a_2$} (3)
	(1) [->,thick] edge[bend right=30] node[below] {$a_2$} (2)
	(1) [<-,thick] edge[bend left=30] node[above] {$a_1$} (2)
	(2) [->,thick] edge[bend right=30] node[below] {$a_2$} (3)
	(2) [<-,thick] edge[bend left=30] node[above] {$a_1$} (3);  
	\end{scope}
	\end{tikzpicture}
	\caption{Deterministic 3-state 2-action MDP with $\sigma^2_1 \!=\! \sigma^2_3 \!=\! 1$ and $\sigma^2_2 \ll 1$.}
	\label{fig:toy_MDP}
\end{figure}
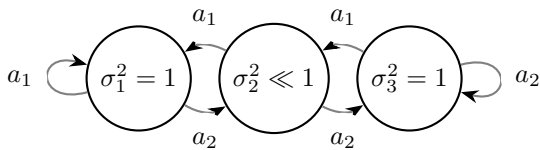


When the budget $n$ is small, the solution of~\eqref{eq:opt.function.asymptotic.lambda} may be very inefficient compared to the optimal finite-time policy. 
As an illustrative example, consider the MDP in Fig.~\ref{fig:toy_MDP}. In the ``unconstrained'' version of the problem, where states can be directly sampled (i.e., the bandit setting), the optimal continuous allocation for  problem~\eqref{eq:opt.function.simplified} tends to $(0.5,0,0.5)$ as $\sigma^2(s_2)$ tends to 0. As soon as we introduce the constraint of the MDP structure, such allocation may not be realizable by any policy. In this MDP, solving problem~\eqref{eq:opt.function.asymptotic.lambda} returns a policy that executes the self-loop actions in $s_1$ and $s_3$ with high probability (thus moving to $s_2$ with low probability) and takes a uniformly random action in $s_2$. The resulting asymptotic performance does indeed approach the optimal unconstrained allocation, as the stationary distribution of the policy $(\eta(s_1), \eta(s_2),\eta(s_3))$ tends to $(0.5,0,0.5)$ for any arbitrary initial state $\overline{s}$. However for any finite budget $n$, this policy performs very poorly since the agent would get stuck in $s_1$ (or $s_3$ depending on the initial state) almost indefinitely, thus making the mean estimation of $s_3$ (or $s_1$) arbitrarily bad. As a result, the optimal asymptotic policy mixes arbitrarily slowly as $\sigma^2(s_2)$ tends to zero and its finite-time performance is then arbitrarily far from the optimal performance. 

This effect is also illustrated by Lem.~\ref{lem:asmyptoric.performance}, where the performance loss of the asymptotic policy depends on $\ell_n(\pi_{\lambda^\star})$, which critically scales with the inverse of the spectral gap $\gamma_{\pi_{\lambda^{\star}}}$. This issue may also significantly affect the performance of \fwmdp, as the gap between $\wh{\psi}^{+}_{k+1}$ and $\wt{\psi}_{k+1}$ may be arbitrarily large if $\wh{\pi}^{+}_{k+1}$ is slowly mixing. This problem together with Lem.~\ref{lem:asmyptoric.performance} suggest regularizing the optimization problems (i.e., problem~\eqref{eq:opt.function.asymptotic} for optimization and the computation of $\wh{\psi}^{+}_{k+1}$ for learning) towards fast mixing policies. 


\textbf{Optimization.} As a direct application of Lem.~\ref{lem:asmyptoric.performance} we could replace problem~\eqref{eq:opt.function.asymptotic} with
\begin{equation}\label{eq:opt.function.regularized}
\begin{aligned}
&\min_{\substack{\pi \in\Pi^{\textrm{SR}}\\ \eta \in \Delta(\calS) }} \mathcal{L}^{\text{reg}}(\pi,\eta) := \mathcal{L}(\pi,\eta) + \ell_n(\pi) \\
&\text{s.t.} \; \forall s \in \calS,\; \eta(s) = \sum_{s',a} \pi(a|s')p(s|s',a)\eta(s') 
\end{aligned}.
\end{equation}
The main advantage of solving this problem is illustrated in the following lemma.

\begin{lemma}\label{lem:reg.performance}
	Let $\pi^\star_\reg$ be the solution of problem~\eqref{eq:opt.function.regularized}, its performance loss is bounded as
	\begin{align}\label{eq:performance.loss.reg}
	\mathcal{L}_n(\pi^\star_\reg) - \mathcal{L}_n(\pi^\star_n) &\leq 2\ell_n(\pi^\star_n).
	\end{align}
\end{lemma}

Since in general we expect $\pi^\star_n$ to mix much faster than $\pi_{\lambda^\star}$ (i.e., $\gamma_{\pi^\star_n} \gg \gamma_{\pi_{\lambda^\star}}$), the performance loss of $\pi^\star_\reg$ may be much smaller than the loss in Lem.~\ref{lem:asmyptoric.performance}. As problem~\eqref{eq:opt.function.regularized} is not convex, we replace it by heuristic convex algorithm. We isolate from $\ell_n(\pi)$ the spectral gap $\gamma_\pi$ and the convergence rate $\rho_n:=S/n$ and, using the norm formulation of the SLEM in \eqref{SLEM_as_norm}, we introduce a proxy to the regularized loss as
\begin{align}\label{eq:loss.proxy}
\!\!\!\!\mathcal{L}(\pi,\eta) + \dfrac{\rho_n}{1-\| D_{\eta}^{1/2} P_\pi D_{\eta}^{-1/2} \!-\! \sqrt{\eta} \sqrt{\eta}^\transp \|_2}.
\end{align}
%
Building on this proxy and the study on computing fastest mixing chains on graphs by~\citet{boyd2004fastest}, we derive \fmh (Faster-Mixing Heuristic) that solves a convex surrogate problem that favors fast mixing policies with limited deviation w.r.t.\ a target stationary distribution. While we postpone the full derivation to App.~\ref{app:ssection:derivation_fmh},
we report the main structure of the algorithm. $\fmh$ receives as input a budget $n$ and the optimal asymptotic policy  $\pi^{\star}$ obtained by solving~\eqref{eq:opt.function.asymptotic}, then it returns a stationary policy $\pi_{\fmh}^{\star}$. The algorithm proceeds through two steps. 

\textit{Step 1 (improvement of the mixing properties).} We first reparametrize the problem by introducing the variable $X \in \Re^{S\times S}$ as $X = D_\eta P_\pi$ and we reduce the difficulty of handling the stationary constraint on $\eta$ by constraining $X$ to respect the adjacency matrix of the MDP $Q$. Notably, we introduce the constraints\footnote{We omit constraints $X \geq 0, \|X\|_1 = 1$ for clarity.}
\begin{align}\label{eq:surrogate.constraints}
X = X^T, \;\; X_{ss'} = 0 ~ \textrm{ if } ~ Q_{ss'} = 0,
\end{align}
which correspond to reversibility and adhering to the ``structure'' of the MDP. Furthermore, since we can recover a state distribution from $X$ as $\eta_X(s) = \sum_{s'} X_{ss'}$, we also need to enforce
\begin{align}\label{eq:surrogate.constraints.bis}
\sum_{s' \in \cal S} \! X_{ss'} \geq \wu{\eta}, \; \sum_{s  \in \cal S}\big(\sum_{s'  \in \cal S} X_{ss'} - \eta_s^{\star}\big)^2 \leq \delta^2_n,
\end{align}
where we lower bound the state distribution and we require $X$ to be close to the target state stationary distribution $\eta^\star$ in $\ell_2$-norm. Since $\mathcal{L}$ is smooth when $\eta$ is lower bounded by $\wu{\eta}$, the $\ell_2$-norm constraint guarantees that the performance of $X$ does not deviate much from $\eta^\star$. \fmh then proceeds by solving 
\begin{equation}\label{eq:opt.surrogate}
\begin{aligned}
& \underset{X}{\min}&&
\!\!\!\!\! \sum_{s \in \mathcal{S}} \dfrac{\sigma^2(s)}{\sum_{s'  \in \mathcal{S}} X_{ss'}}  \\
 &&&\!\!\!\!\!\!\!\!\!\!\!\!+ \dfrac{\rho_n}{1-\| D_{\eta^{\star}}^{-1/2} X D_{\eta^{\star}}^{-1/2} - \sqrt{\eta^{\star}} \sqrt{\eta^{\star}}^T \|_2} \\
& \text{s.t.}& &~\eqref{eq:surrogate.constraints}, \eqref{eq:surrogate.constraints.bis}
\end{aligned}.
\end{equation}
Unlike the proxy loss~\eqref{eq:loss.proxy}, this problem is convex in $X$ and can be solved using standard convex optimization tools.

\textit{Step 2 (projection onto the set of feasible stationary policies).} Unfortunately $\eta_X(s) = \sum_{s'} X_{ss'}$ may not be feasible in the MDP (i.e., it may not be stationary). Thus we finally proceed with the computation of a policy $\pi$ whose stationary distribution is closest to $\eta$ by solving the convex problem
\begin{equation*}
\begin{aligned}
& \underset{\pi}{\min} \quad
\displaystyle \sum_{s \in \cal{S}} \Big(\eta_X(s) -\!\!\!\!\! \displaystyle \sum_{s' \in \calS, a \in \A_{s'}}\!\!\!\!\! \eta_X(s')p(s|s',a)\pi_{s',a}  \Big)^2  \\
& \text{s.t. \quad $\pi_{s,a} \geq 0 \quad \textrm{and} \quad \displaystyle \sum_{a \in \mathcal{A}_s}\pi_{s,a} = 1$}. \nonumber
\end{aligned}
\end{equation*}
\fmh thus returns a policy that may have better mixing properties than $\pi^{\star}$ at the cost of a slight loss in asymptotic performance. The performance loss of \fmh approaches the one of $\pi^\star_\reg$ as shown in the following lemma.
	
\begin{lemma}\label{lem:performance_loss_fmh}
Let $\pi^{\star}_{\fmh}$ be the policy returned by \fmh, then the performance loss is bounded as
\begin{align*}
&\mathcal{L}_n(\pi^{\star}_{\fmh}) - \mathcal{L}_n(\pi_n^{\star})\\
& \leq 2 \ell_n(\pi_n^{\star}) + \dfrac{2\sigma_{\max}^2 \sqrt{S}}{\wu{\eta}^2} \delta_n + \dfrac{2}{\gamma_{\min}}\rho_n + O(n^{-3/2}).
\end{align*}
\end{lemma}

This suggests that the slack variable $\delta_n$ should decrease at least as $O(n^{-1})$ to guarantee the algorithm's consistency and not worsen the overall performance. 

Finally, we introduce in App.~\ref{app:ssection:SDP_variant} a more computationally efficient variant of step 1 of \fmh that uses semidefinite programming, which we later refer to as \fmhsdp.

\textbf{Learning.} 
As discussed above and shown in the proof of Thm.~\ref{thm:regret_FWAME}, the regret of \fwmdp depends on the mixing properties of the policy $\wh{\pi}^{+}_{k+1}$.
While the optimization problem to compute $\wh{\psi}^{+}_{k+1}$ is different than problem~\eqref{eq:opt.function.asymptotic}, the surrogate optimization procedure described above can be readily applied to this case as well. In fact, $\eta^\star$ received in input is now the target state-action stationary distribution $\wh{\psi}^{+}_{k+1}$ and, since the objective function is still smooth, the deviation constraint does limit the performance loss that could be incurred because of the deviation $\delta_n$. App.~\ref{FWAME_FMHSDP} provides more discussion on the resulting learning algorithm that we call \fwmdp w/ \fmhsdp.

\vspace{-0.1in}
\section{Numerical Simulations}\label{sec:experiments}
\vspace{-0.1in}

\begin{figure*}
	\begin{minipage}[c]{0.46\textwidth}
		\centering
			\includegraphics[width=0.8\linewidth]{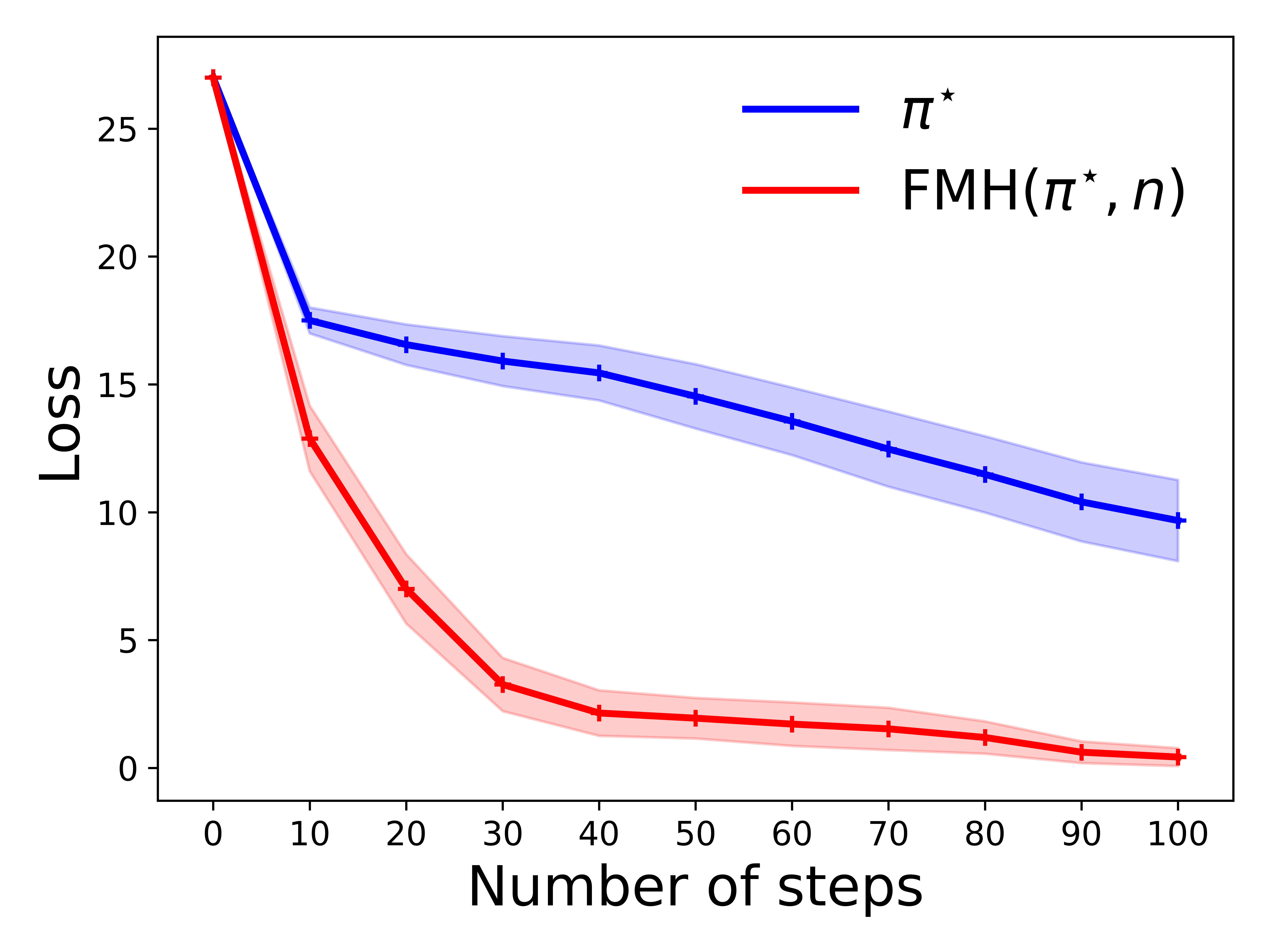}
			\vspace{-0.1in}
		\caption{$\loss(\pi,n,R=100)$ as a function of $n$ in the 3-state MDP of Fig.~\ref{fig:toy_MDP} (with $\sigma^2_2 = 0.001$).}
		\label{plot_mixing_3state}
	\end{minipage}
\hfill
  \begin{minipage}[c]{0.52\textwidth}
		\centering
		\begin{tabular}{|c|c|c||c|c|}
			\hline
			$\pi$ & \multicolumn{2}{c||}{$\pi_{\fwmdp}$} & \multicolumn{2}{c|}{$\pi_{\textrm{unif}}$} \\ \hline    
			$n$ & $500$  &
			$1000$ & $500$ &  $1000$ 
			\\  \hhline{|=|=|=|=|=|}
			${\cal G}_{S=5}$ & 0.31 & 0.10 & 2.18 & 1.04         \\ \hline
			${\cal G}_{S=10}$ & 0.35 & 0.19 & 1.98 & 1.15 \\ \hline
		\end{tabular}
		\caption{$\ratio(\pi,n,R=100)$ for $n \in \{500, 1000\}$ and for $\pi_{\fwmdp}$ and $\pi_{\textrm{unif}}$, averaged over 100 Garnet instances randomly generated from ${\cal G}(S, A=3, b=2)$ for $S \in \{5,10\}$.}
		\label{table_rnl}
	\end{minipage}
	
\end{figure*}	

%

\begin{figure*}[!t] 
\begin{subfigure}[t]{0.49\textwidth}
\centering
\includegraphics[trim={0.2cm 0 0.2cm 0},width=0.95\linewidth]{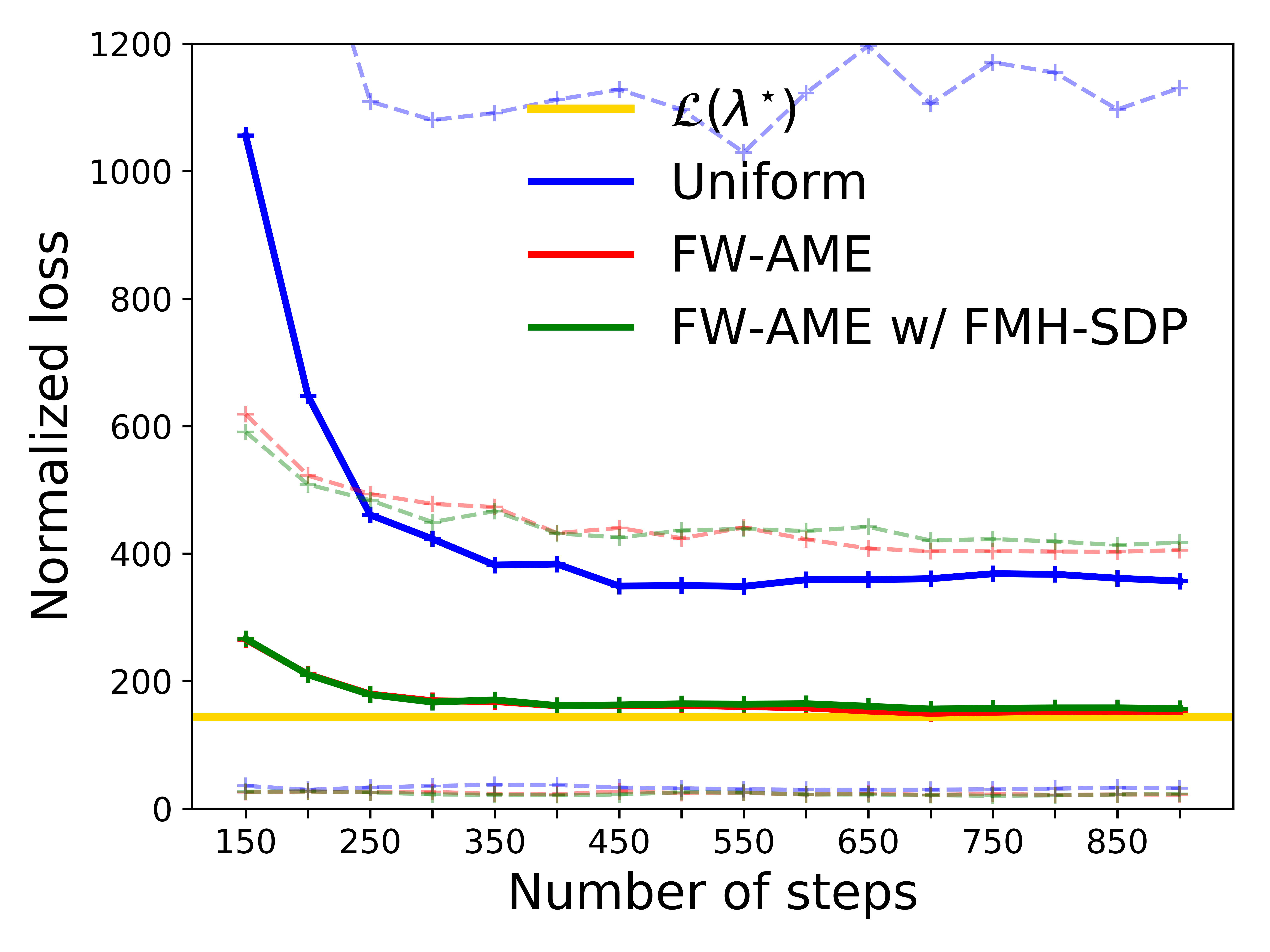}
\caption{An instance of ${\cal G}_{\cal R}(S=5, A=3, b=3)$ with fast mixing policies. The average SLEM is roughly 0.55, w/ or w/o \fmhsdp. 
}
\label{plot_learning_mixing_1}
\end{subfigure}\hspace*{\fill}
\begin{subfigure}[t]{0.49\textwidth}
\centering
\includegraphics[trim={0.2cm 0 0.2cm 0},clip,width=0.95\linewidth]{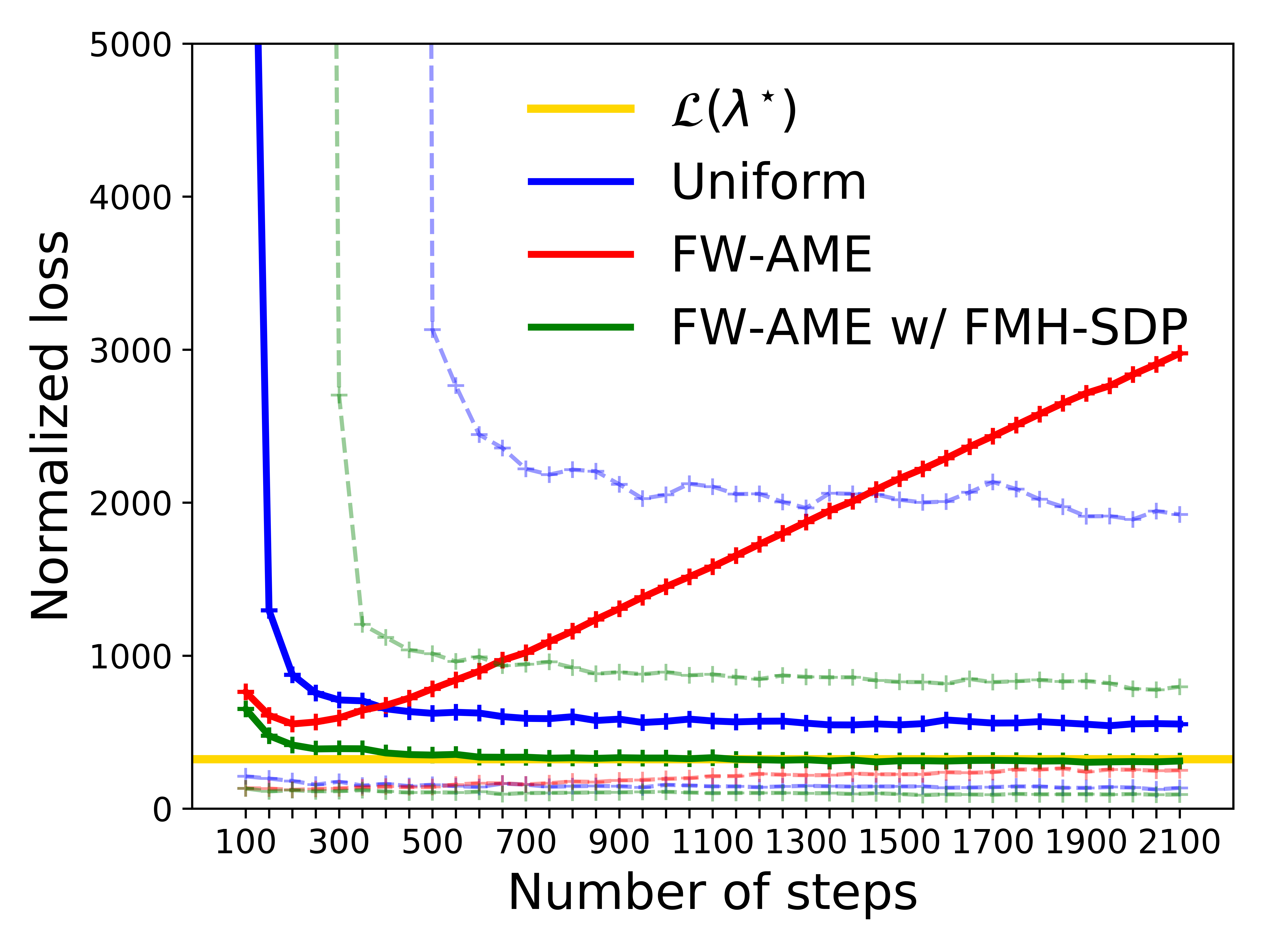}
\caption{An instance of ${\cal G}_{\cal R}(S=10, A=2, b=2)$ where policies mix poorly. The  average SLEM is 0.95 and it is decreased to 0.88 by \fmhsdp. 
}
\label{plot_learning_mixing_2}
\end{subfigure}
\caption{$n\loss(\pi,n,R=1000)$ as a function of $n$. The dashed curves report 5\% and 95\% quantiles.}
\label{plot_learning_mixing}
\end{figure*}

\textbf{Experimental settings.} We consider $\nu(s) = \mathcal{N}(0, \sigma^2(s))$ and when $T(s) = 0$, we set default variance and mean predictions to $\sigma_{\max}^2$ and $3 \sigma_{\max}$. The initial state is drawn uniformly at random from $\calS$. The episodes of \fwmdp are set so that $t_k = \tau_1 + (k-1)^3$ (for $k > 1$, otherwise $t_1=1$), where $\tau_1$ is the (adaptive) time needed for the initial policy to collect at least one sample of each state (so as to satisfy Asm.~\ref{asm:counter} after the first episode). 
We set $\wu \eta = 0.001$ and the confidence intervals to $\alpha(t,s,1/t) = 0.2 \sigma^2_{\textrm{max}}\sqrt{\log(4St^2)/T_t(s)}$. 
%
%
We run simulations on a set of random Garnet MDPs~\citep{bhatnagar2009natural}. A Garnet instance ${\cal G}(S, A, b, \sigma_{\min}^2, \sigma_{\max}^2)$ has $S$ states, $A$ actions, $b$ is the branching factor and state variances are random in $[\sigma_{\min}^2, \sigma_{\max}^2]$. ${\cal G}_{\cal R}$ denotes the reversible Garnet MDPs (see App.~\ref{app:experiments} for more details). We set $\sigma^2_{\textrm{min}}=0.01$ and $\sigma^2_{\textrm{max}}=10$ to have a large spread between the state variances.
%
For any budget $n$ and policy $\pi$ ran over $R$ runs, the estimation loss is
    \begin{align*}
        \loss(\pi,n,R) = \dfrac{1}{S R} \displaystyle \sum_{s\in \mathcal{S}} \displaystyle \sum_{1\leq r \leq R} \Big[ \big( \widehat{\mu}_{\pi,n}^{(r)}(s) - \mu(s)\big)^2\Big],
    \end{align*}
while the normalized loss is $n \loss(\pi,n,R)$. Finally, we measure the competitive ratio w.r.t.~the optimal asymptotic performance as
\begin{align}\label{rnl_formula}
        \ratio(\pi,n,R) = \dfrac{n\loss(\pi,n,R)}{\mathcal{L}(\lambda^{\star})} - 1.
\end{align}

\textbf{Results.} We first verify the regret guarantees of Thm.~\ref{thm:regret_FWAME}. Fig.~\ref{table_rnl} reports the competitive ratio averaged over 100 randomly generated Garnet MDPs for \fwmdp and a uniform policy $\pi_{\textrm{unif}}(a|s) = 1/\abs{\mathcal{A}_s}$. As expected the ratio (which is a proxy for the regret) of \fwmdp is much smaller than for $\pi_{\textrm{unif}}$ and it approaches zero as the budget increases. While we report only the aggregated values, this result is consistently confirmed across all Garnet instances we have tried.

We then study the effectiveness of \fmh in improving the optimization performance. In Fig.~\ref{plot_mixing_3state} we report $\loss(\pi,n)$ for the asymptotic optimal policy $\pi_{\lambda^\star}$ and the surrogate policy $\pi^\star_{\fmh}$ as a function of $n$ in the simple 3-state MDP illustrated in Fig.~\ref{fig:toy_MDP}, where $\pi^{\star}$ mixes poorly. We notice that in this case, the impact of favoring faster mixing policies does translate to a significant improvement in finite-time performance. This finding is also confirmed when \fmh is applied to \fwmdp. We first show a specific reversible Garnet MDP where all the policies generated by \fmh are mixing relatively fast (see the normalized loss in Fig.~\ref{plot_learning_mixing_1}). In this case, \fmhsdp has the same performance as \fwmdp (and both are significantly better than uniform). This is confirmed by evaluating the average SLEM of the policies generated by the two algorithms, which is roughly 0.55 in both cases. 
On the other hand, there are Garnet MDP instances where \fwmdp may indeed generate very poorly mixing policies that are executed for relatively long episodes, thus compromising the performance of the algorithm (see Fig.~\ref{plot_learning_mixing_2}).\footnote{The algorithm is still able to recover from bad mixing policies thanks to ergodicity and changing episodes, but it takes much longer to converge.} In this case, \fmhsdp successfully biases the learning process towards faster mixing policies and obtains a much better finite-time performance. In fact, the average SLEM of the policies generated \fwmdp is successfully reduced from 0.95 to 0.88 for \fmhsdp.


\vspace{-0.1in}
\section{Conclusion and Extensions}\label{sec:conclusion}
\vspace{-0.1in}


We introduced the problem of active exploration in MDPs, proposed an algorithm with vanishing regret and proposed a heuristic convex optimization problem to favor fast mixing policies. This paper opens a number of questions: \textit{(1)} A lower bound is needed to determine the complexity of active exploration in MDPs compared to the \mab case; \textit{(2)} While the ergodicity assumption is not needed in regret minimization in MDPs~\citep{jaksch2010near}, it is unclear whether it is mandatory in our setting; \textit{(3)} A full regret analysis of the case of unknown MDP (see App.~\ref{app:relaxing:asm:known.mdp}). This paper may be a first step towards formalizing the problem of intrinsically motivated RL, where the implicit objective is often to accurately estimate the MDP dynamics and effectively navigate through states \cite[see e.g.,][]{auer2011models,maxent2018}.

\vskip 0.2in
\newpage
\bibliography{biblio}

\begin{thebibliography}{}

\bibitem[Akshay et~al., 2013]{Akshay}
Akshay, S., Bertrand, N., Haddad, S., and Helouet, L. (2013).
\newblock The steady-state control problem for markov decision processes.
\newblock In {\em International Conference on Quantitative Evaluation of
  Systems}, pages 290--304.

\bibitem[Antos et~al., 2010]{antos2010active}
Antos, A., Grover, V., and Szepesv{\'a}ri, C. (2010).
\newblock Active learning in heteroscedastic noise.
\newblock {\em Theoretical Computer Science}, 411(29-30):2712--2728.

\bibitem[Auer et~al., 2011]{auer2011models}
Auer, P., Lim, S.~H., and Watkins, C. (2011).
\newblock Models for autonomously motivated exploration in reinforcement
  learning.
\newblock In {\em International Conference on Algorithmic Learning Theory},
  pages 14--17.

\bibitem[Balcan et~al., 2015]{balcan2015commitment}
Balcan, M.-F., Blum, A., Haghtalab, N., and Procaccia, A.~D. (2015).
\newblock Commitment without regrets: Online learning in stackelberg security
  games.
\newblock In {\em Proceedings of the sixteenth ACM conference on economics and
  computation}, pages 61--78.

\bibitem[Basilico et~al., 2012]{basilico2012patrolling}
Basilico, N., Gatti, N., and Amigoni, F. (2012).
\newblock Patrolling security games: Definition and algorithms for solving
  large instances with single patroller and single intruder.
\newblock {\em Artificial Intelligence}, 184:78--123.

\bibitem[Berthet and Perchet, 2017]{berthet2017fast}
Berthet, Q. and Perchet, V. (2017).
\newblock Fast rates for bandit optimization with upper-confidence frank-wolfe.
\newblock In {\em Advances in Neural Information Processing Systems}, pages
  2225--2234.

\bibitem[Bhatnagar et~al., 2009]{bhatnagar2009natural}
Bhatnagar, S., Sutton, R., Ghavamzadeh, M., and Lee, M. (2009).
\newblock Natural actor-critic algorithms.
\newblock {\em Automatica}, 45(11):2471--2482.

\bibitem[Boyd et~al., 2004]{boyd2004fastest}
Boyd, S., Diaconis, P., and Xiao, L. (2004).
\newblock Fastest mixing markov chain on a graph.
\newblock {\em SIAM review}, 46(4):667--689.

\bibitem[Carpentier et~al., 2011]{carpentier2011upper}
Carpentier, A., Lazaric, A., Ghavamzadeh, M., Munos, R., and Auer, P. (2011).
\newblock Upper-confidence-bound algorithms for active learning in multi-armed
  bandits.
\newblock In {\em International Conference on Algorithmic Learning Theory},
  pages 189--203.

\bibitem[Chentanez et~al., 2005]{chentanez2005intrinsically}
Chentanez, N., Barto, A.~G., and Singh, S.~P. (2005).
\newblock Intrinsically motivated reinforcement learning.
\newblock In {\em Advances in Neural Information Processing Systems 17}, pages
  1281--1288.

\bibitem[Dance and Silander, 2017]{dance2017optimal}
Dance, C.~R. and Silander, T. (2017).
\newblock Optimal policies for observing time series and related restless
  bandit problems.
\newblock {\em arXiv preprint arXiv:1703.10010}.

\bibitem[Diaconis et~al., 1991]{diaconis1991geometric}
Diaconis, P., Stroock, D., et~al. (1991).
\newblock Geometric bounds for eigenvalues of markov chains.
\newblock {\em The Annals of Applied Probability}, 1(1):36--61.

\bibitem[Hanneke, 2014]{hanneke2014theory}
Hanneke, S. (2014).
\newblock Theory of disagreement-based active learning.
\newblock {\em Foundations and Trends in Machine Learning}, 7(2-3):131--309.

\bibitem[Hazan et~al., 2018]{maxent2018}
Hazan, E., Kakade, S.~M., Singh, K., and Soest, A.~V. (2018).
\newblock Provably efficient maximum entropy exploration.
\newblock {\em CoRR}, abs/1812.02690.

\bibitem[Hsu et~al., 2015]{Hsu2015mixing}
Hsu, D.~J., Kontorovich, A., and Szepesv{\'a}ri, C. (2015).
\newblock Mixing time estimation in reversible markov chains from a single
  sample path.
\newblock In {\em Advances in neural information processing systems}, pages
  1459--1467.

\bibitem[Jaggi, 2013]{Jaggi}
Jaggi, M. (2013).
\newblock Revisiting frank-wolfe: Projection-free sparse convex optimization.
\newblock In {\em Proceedings of The 30th International Conference on Machine
  Learning}, volume~28, pages 427--435.

\bibitem[Jaksch et~al., 2010]{jaksch2010near}
Jaksch, T., Ortner, R., and Auer, P. (2010).
\newblock Near-optimal regret bounds for reinforcement learning.
\newblock {\em Journal of Machine Learning Research}, 11(Apr):1563--1600.

\bibitem[Lattimore and Szepesv\'{a}ri, 2019]{lattimore2018bandit}
Lattimore, T. and Szepesv\'{a}ri, C. (2019).
\newblock {\em Bandit Algorithms}.
\newblock Cambridge University Press (preprint).

\bibitem[Paulin et~al., 2015]{Paulin}
Paulin, D. et~al. (2015).
\newblock Concentration inequalities for markov chains by marton couplings and
  spectral methods.
\newblock {\em Electronic Journal of Probability}, 20.

\bibitem[Pukelsheim, 2006]{pukelsheim2006optimal}
Pukelsheim, F. (2006).
\newblock {\em Optimal Design of Experiments}.
\newblock Classics in Applied Mathematics. Society for Industrial and Applied
  Mathematics.

\bibitem[Puterman, 1994]{puterman1994markov}
Puterman, M.~L. (1994).
\newblock {\em Markov Decision Processes: Discrete Stochastic Dynamic
  Programming}.
\newblock John Wiley \& Sons, Inc., New York, NY, USA.

\bibitem[Rolf et~al., 2018]{rolf2018a-successive-elimination}
Rolf, E., Fridovich{-}Keil, D., Simchowitz, M., Recht, B., and Tomlin, C.
  (2018).
\newblock A successive-elimination approach to adaptive robotic sensing.
\newblock {\em CoRR}, abs/1809.10611.

\bibitem[Welch, 1982]{welch1982algorithmic}
Welch, W.~J. (1982).
\newblock Algorithmic complexity: three np-hard problems in computational
  statistics.
\newblock {\em Journal of Statistical Computation and Simulation},
  15(1):17--25.

\end{thebibliography}

\newpage
\begin{appendices}
\onecolumn

\section{Relaxing assumptions}\label{app:relaxing_asm}

In this section we review the assumptions used throughout the paper and discuss if and how they could be relaxed.

\subsection{Assumption \ref{asm:counter}} \label{app:relaxing:asm:counter}

We consider how to remove Asm.~\ref{asm:counter}. When $T_{\pi,t}(s) = 0$ we set $\wh{\mu}_{\pi,t}(s)$ to an arbitrary default value $\mu_{\infty}$.\footnote{Formally $\mu_{\infty} = +\infty$ yet we can also set it equal to a suitable finite value depending on the distributions. For example, if the state distributions are Gaussian and the means belong to an interval $[-\mu_{\textrm{max}}, +\mu_{\textrm{max}}]$, we can fix $\mu_{\infty} = 3\sigma_{\textrm{max}} + \mu_{\textrm{max}}$, which ensures that the mean estimate computed from one single sample has an overwhelming probability of being more accurate than the default value $\mu_{\infty}$ when there are no samples.} In this case, the prediction loss becomes
\begin{align*}
    \mathcal{L}_n(\pi) = \dfrac{n}{S}\sum_{s \in \mathcal{S}} \mathbb{E}_{\pi} \big[ \dfrac{\sigma^2(s)}{T_{\pi,n}(s)} | T_{\pi,n}(s) > 0 \big] + E(\pi, n), ~~ \textrm{with}~ E(\pi, n) := \dfrac{n}{S}\sum_{s \in \mathcal{S}}(\mu_{\infty} - \mu(s))^2 \mathbb{P}(T_{\pi,n}(s)=0).
\end{align*}
Asm.~\ref{asm:counter} makes the simplification that $E(\pi, n)=0$. In order to deal with the general case, we need to take care of the event $\{ \exists s \in \mathcal{S}, T_{\pi,n}(s) = 0 \}$ in which at least one state does not have any sample from which we could estimate its mean. An alternative is to consider that we initially have a ``fictitious'' observation equal to a fixed value at each state, which would introduce a small bias that tends to zero quickly. Another alternative could be to start by running a policy $\pi_0$ over the states of the MDP and as soon as each state is visited at least once, we set the time step equal to 1 and begin our analysis. In the framework of the learning algorithm \fwmdp, Asm.~\ref{asm:counter} can be easily replaced in practice by considering an adaptive length $\tau_1$ such that at least one sample of each state is collected at the end of the first episode (which is what we do in the experiments in Sect.~\ref{sec:experiments}). The length of this phase would be small as the following result applies.
\begin{proposition} For any policy $\pi \in \Pi^{\textrm{SR}}$, under Asm.~\ref{asm:ergodic}, the term $E(\pi, n)$ decreases exponentially in $n$.
\label{no_visit}
\end{proposition}
\begin{proof}
Let $n > 1/\eta_{\pi, \min}$. Then setting $\epsilon = \eta(s) - 1/n > 0$ yields\footnote{Here we use the more general result for non-reversible chains reported in Prop.~\ref{p:non.reversible.convergence}.}
\begin{align*}
\mathbb{P}\big(T_{\pi,n}(s) = 0\big) &= \mathbb{P}\Big(T_{\pi,n}(s) < (\eta(s)-\epsilon)n\Big) \\
&\leq \sqrt{\frac{2}{\eta_{\pi, \min}}}  \exp\bigg( \dfrac{-n \gamma_{\textrm{ps}}^{\pi} \Big(\eta(s) - \dfrac{1}{n}\Big)^2}{ 16\eta(s)(1-\eta(s))\Big(1 + \dfrac{1}{n\gamma_{\textrm{ps}}^{\pi}}\Big) + 40 \Big(\eta(s) - \dfrac{1}{n}\Big)}\bigg).
\end{align*}
We thus obtain for any stationary policy $\pi$ and any budget $n > 1/\eta_{\pi, \min}$
\begin{align*}
E(\pi, n) \leq \dfrac{n}{S} \sum_{s \in \mathcal{S}}(\mu_{\infty} - \mu(s))^2 \sqrt{\frac{2}{\eta_{\pi, \min}}}  \exp\bigg( \dfrac{-n \gamma_{\textrm{ps}}^{\pi} \Big(\eta_{\pi, \min} - \dfrac{1}{n}\Big)^2}{ 8\Big(1 + \dfrac{1}{n\gamma_{\textrm{ps}}^{\pi}}\Big) + 40 \Big(\eta_{\pi, \max} - \dfrac{1}{n}\Big)} \bigg),
\end{align*}
which proves the result.
\end{proof}

\subsection{Assumption \ref{asm:reversible}}
\label{app:relaxing:asm:reversible}

The reversibility assumption (Asm.~\ref{asm:reversible}) can be removed and Prop.~\ref{prop.concentration.mixing}, Lem.~\ref{lem:asmyptoric.performance} as well as the proof of Thm.~\ref{thm:regret_FWAME} could be easily adjusted to handle the case of non-reversible policies. As a result, the reversibility condition does not need to hold for the algorithm \fwmdp and its vanishing regret guarantees. This can be achieved by replacing Prop.~\ref{prop.concentration.mixing} with a concentration result adapted from~\cite{Paulin}.

\begin{proposition}[Thm.~3.10 and Prop.~3.14 from~\cite{Paulin}]
\label{p:non.reversible.convergence}
Let us fix a stationary policy $\pi$ which induces a time-homogeneous, ergodic Markov chain. We denote by $P$ its transition matrix and by $\hat{P}$ the time-reversal matrix of $P$. We denote by $\eta_{\pi, \min} = \min_{s \in \mathcal{S}}\eta(s) > 0$ and $\eta_{\pi, \max} = \max_{s \in \mathcal{S}}\eta(s)$ where $\eta$ is the chain's stationary distribution. We consider the pseudo-spectral gap $\gamma_{\textrm{ps}}^{\pi} = \max_{k \geq 1}\gamma(\hat{P}^k P^k)/k > 0$. For a given state $s$ and for every $\epsilon > 0$
\begin{align*}
\mathbb{P}\Big(|\frac{T_{\pi,n}(s)}{n} - \eta(s)| > \epsilon\Big) \leq \sqrt{\frac{2}{\eta_{\pi, \min}}}  \exp{\bigg( \dfrac{-n \gamma_{\textrm{ps}}^{\pi} \epsilon^2}{ 16\eta(s)(1-\eta(s))\big(1 + 1/(n\gamma_{\textrm{ps}}^{\pi})\big) + 40 \epsilon}\bigg)}.
\end{align*}
\end{proposition}

In Sect.~\ref{sec:relaxation}, the reversibility condition is intrinsically needed to relate the spectral gap with its spectral norm formulation, which is not possible for the pseudo-spectral gap. Nonetheless, rather than assuming that all policies are in the set of reversible stationary randomized policies $\Pi^{\textrm{SRR}}$, we could focus on computing a policy $\pi_{\fmh}^{\star}$ belonging to the restricted set $\Pi^{\textrm{SRR}}$, thus replacing the assumption with an additional constrained in the optimization problem.

\subsection{Assumption \ref{asm:known.mdp}}
\label{app:relaxing:asm:known.mdp}

We can deal with the case when the MDP transition model $p$ is unknown by following an optimistic approach similar to UCRL~\citep{jaksch2010near}. We recall that the optimization problem solved by \fwmdp at each episode is indeed equivalent to solving an MDP with known $p$ and reward function set to $\nabla \wh{\mathcal{L}}_{t_k-1}^{+}(\wt{\lambda}_k)$, which is already an optimistic evaluation of the true gradient. Whenever $p$ is unknown, but an estimate and a confidence set are available, we can include the uncertainty of the estimate of $p$ into the optimistic optimization of the MDP.
Let us fix an episode $k$ and $t=t_k-1$ the time step at the end of the previous episode. We introduce the following set that is $p$-dependent and thus unknown
\begin{align*}
    \Lambda_{\wu{\eta}}^{(p)} = \Big\{ \lambda \in \Delta(\mathcal{S} \times \mathcal{A}) ~:~ \forall s \in \mathcal{S}, ~ \displaystyle\sum_{b \in \mathcal{A}} \lambda(s,b) \geq 2\wu{\eta} \quad \textrm{and} \quad \displaystyle\sum_{b \in \mathcal{A}} \lambda(s,b) = \displaystyle\sum_{s' \in \mathcal{S}, a \in \mathcal{A}}p(s|s',a)\lambda(s',a) \Big\}.
\end{align*}
The aim is the solve the following problem
\[
    \min_{\lambda \in \Lambda_{\wu{\eta}}^{(p)}} \min_{p \in \mathcal{C}_t} \langle c, \lambda \rangle = \sum_{s,a} \nabla \wh{\mathcal{L}}_{t}^{+}(\wt{\lambda}_k)(s,a) \lambda(s,a).
\]
If we define over $\mathcal{S} \times \mathcal{A}$ the (bounded) reward function $r= -\nabla \wh{\mathcal{L}}_{t}^{+}(\wt{\lambda}_k)$, we notice that the above problem can be reduced to the dual formulation of finding the policy that maximizes the average reward \citep[Sect.~8,][]{puterman1994markov}. As such, it becomes equivalent to solving the following problem
\begin{align}
\label{dual_pbm}
\max_{\pi \in \Lambda_{\wu{\eta}}^{(p)}} \max_{p \in \mathcal{C}_t} \rho_{\pi}(p),
\end{align}
where $\rho_{\pi}(p)$ is the gain of stationary policy $\pi$ in the MDP with transition probability function $p$. The confidence set $\mathcal{C}_t$ defines a set of plausible transition probability functions at time $t$. Since the reward function is known, this corresponds to a set of plausible MDPs. Problem (\ref{dual_pbm}) thus returns an optimal policy in the plausible MDP with the largest gain. \citet[Sect.~38,][]{lattimore2018bandit} explicit the construction of $\mathcal{C}_t$ and explain that the solutions of (\ref{dual_pbm}) are guaranteed to exist and can be found efficiently.\footnote{In a nutshell, the justification comes from introducing the extended Markov decision process $\wt{M}$ from \cite{jaksch2010near} and solving the average reward problem on that specific MDP using Extended Value Iteration. The fact that the extended action-sets of $\wt{M}$ are infinite is not problematic since $\mathcal{C}_t$ is a convex polytope and has finitely many extremal points; as a result restricting the confidence sets to these points makes the extended MDP finite without changing the optimal policy.}

While a complete derivation of the regret bound for this algorithm is left for future work, we expect the final result of Thm.~\ref{thm:regret_FWAME} to remain unchanged. In fact, the optimal $p$ returned by problem (\ref{dual_pbm}) belongs to $\mathcal{C}_t$ so it is close to the real $p$ up to a factor scaling in $1/T_t$ by construction of $\mathcal{C}_t$. Hence, if the number of visits of any state-action pair (and not just the number of any state visit as in the case of known $p$) is enforced to be proportional to the time step with high probability, then the derivation of the $\wt O(t^{-1/3})$ rate in the proof of Prop.~\ref{thm:regret_FWAME} (cf.~App.~\ref{app:proof_regret}) is unchanged.

\begin{algorithm}[t]
	\caption{\fmh}
	\label{alg:FMH}
	\begin{algorithmic} 
		\REQUIRE $\eta^{\star}$ is the optimal stationary distribution of the convex problem (\ref{eq:opt.function.asymptotic.lambda}).\\
		\REQUIRE 3 parameters $\rho_n$, $\delta_n$ and $\wu{\eta}$ (typically set respectively to $S / n, 1/n$ and $\min_{s} \eta^{\star}(s) / 2$).\\
		\STATE Compute $X_1$ the optimal solution of the convex problem $(\mathcal{P}_1)$ with parameters $\rho_n$, $\delta_n$ and $\wu{\eta}$.\\
		\STATE Deduce the corresponding state distribution $\eta_1$: $\eta_1(s) = \sum_{s' \in \mathcal{S}} X(s,s')$.\\
		\STATE Compute the optimal stationary policy of the convex problem $(\mathcal{P}_2)$ with $\eta_1$ as target state distribution.
	\end{algorithmic}
\end{algorithm}

\section{Faster-Mixing Heuristic \fmh}\label{app:fast_mixing_heuristic}


\subsection{Derivation of the two-step method}\label{app:ssection:derivation_fmh}

$\fmh(\pi^{\star},n)$ receives as input a budget $n$ and $\pi^{\star}$, the optimal solution of (\ref{eq:opt.function.asymptotic.lambda}), and returns a stationary policy $\pi_{\fmh}^{\star}$ by solving two convex optimization problems. An outline of \fmh is provided in Alg.~\ref{alg:FMH}.

\textbf{Step 1.} In this step we first remove the stationarity constraint on $\eta$ w.r.t.~the MDP dynamics and replace it by a weaker but easier constraint involving the adjacency matrix of the MDP. Instead of using $P$ as the kernel of the Markov chain associated to a policy, we consider it as a generic \textit{transition matrix} that respects the possible transitions in the MDP, i.e., $P_{ij} = 0$ if $Q_{ij} = 0$. In this case, problem \eqref{eq:opt.function.regularized} becomes convex in $P$ for a fixed $\eta$ and convex in $\eta$ for a fixed $P$, yet it is non-convex in both $P$ and $\eta$. When $P$ is fixed, $\eta$ has no more degree of freedom (i.e., it can be directly derived from $P$), thus any framework of alternate minimization cannot be applied here. We notice that the constraint of reversibility $D_{\eta} P = P^T D_{\eta}$ is the toughest one to handle, since it involves both $P$ and $\eta$ and is not convex in $P$ and $\eta$. This leads us to introduce the matrix variable $X = D_{\eta} P \in \mathbb{R}^{S \times S}$ (i.e., $X_{ij} = \eta_i P_{ij}$). The reversibility constraint on $P$ thus simply translates to a symmetry constraint on $X$. More discussion on the characteristics of the matrix $X$ is for example provided in~\cite{Hsu2015mixing}. We then obtain the following optimization problem with variable $X$ (and its corresponding $\eta$)
\begin{align}
& \underset{X, ~ \eta}{\text{minimize}} & & \mathcal{L}_0(X) ~:=~ \sum_{s \in \mathcal{S}} \dfrac{\sigma^2(s)}{\eta(s)}  + \rho_n \dfrac{1}{1-\| D_{\eta}^{-1/2} X D_{\eta}^{-1/2} - \sqrt{\eta} \sqrt{\eta}^T \|_2} \\
& \text{subject to}
& & X \geq 0, \quad  X = X^T, \quad \sum_{j \in \mathcal{S}} X_{ij} = \eta_i ~ \forall i \in \mathcal{S},  \quad X_{ij} = 0 ~ \textrm{if} ~ Q_{ij} = 0, \quad \eta \geq \wu{\eta}, \quad \eta ^T \mathds{1} = 1. \nonumber
\end{align}

This problem is still non-convex in $X$ and $\eta$. An idea could be to fix $\eta$ and solve the convex problem in $X$ (or equivalently $P$). The most straightforward choice for $\eta$ is to use $\eta^{\star}$, the optimal stationary distribution of problem (\ref{eq:opt.function.asymptotic.lambda}), and solve the convex problem of finding the fastest mixing Markov chain with stationary distribution $\eta^{\star}$~\citep{boyd2004fastest}. However the Markov chains whose stationary distributions are $\eta^{\star}$ might all mix poorly. Leveraging the intuition behind the regularized problem~\eqref{eq:opt.function.regularized}, we give more slack to $\eta$ in order to find faster mixing Markov chains, at the cost of having $\mathcal{L}(\eta)$ slightly larger than $\mathcal{L}(\eta^{\star})$, i.e., at the cost of a slightly worse asymptotic performance. We formalize this trade-off with the a parameter $\delta_n$, which represents how close we allow the stationary distribution $\eta$ to be to $\eta^{\star}$ with respect to the $\ell_2$-norm (we pick the $\ell_2$-norm in order to ensure the convexity of the resulting constraint). We thus focus on solving the following surrogate optimization problem $(\mathcal{P}_1)$
\begin{align}
& \underset{X}{\text{minimize}}
& & \mathcal{L}_1(X) ~:=~ \sum_{s \in \mathcal{S}} \dfrac{\sigma^2(s)}{\sum_{j  \in \mathcal{S}} X_{sj}}  + \rho_n \dfrac{1}{1-\| D_{\eta^{\star}}^{-1/2} X D_{\eta^{\star}}^{-1/2} - \sqrt{\eta^{\star}} \sqrt{\eta^{\star}}^T \|_2}  \tag{$\mathcal{P}_1$} \\
& \text{subject to}
& & X \geq 0, \quad  X = X^T, \quad X_{ij} = 0 ~ \textrm{if} ~ Q_{ij} = 0, \nonumber\\
&&& \| D_{\eta^{\star}}^{-1/2} X D_{\eta^{\star}}^{-1/2} - \sqrt{\eta^{\star}} \sqrt{\eta^{\star}}^T \|_2 \leq 1, \nonumber \\
&&& \sum_{(i,j) \in \mathcal{S}^2} X_{ij} = 1, \quad \sum_{j \in \mathcal{S}} X_{ij} \geq \wu{\eta}, \quad \sum_{i \in \mathcal{S}}\big(\sum_{j \in \mathcal{S}} X_{ij} - \eta_i^{\star}\big)^2 \leq \delta_n^2,\nonumber
\end{align}
where the small positive constant $\wu{\eta}$ should satisfy $\wu{\eta} \leq \min_{s} \eta^{\star}(s)$. Prop.~\ref{well-defined} guarantees the convexity and feasibility of the optimization problem $(\mathcal{P}_1)$.

\begin{proposition}
	\label{well-defined}
	$(\mathcal{P}_1)$ is convex in $X$ and well-defined for any $\delta_n$.
\end{proposition}

\begin{proof}
The convexity of $(\mathcal{P}_1)$ is easily obtained from the convexity of the non-regularized problem, the convexity of the function $X \mapsto 1/(1-\|X\|_2)$ and the convexity of all the constraints. There can however exist some matrices $X$ such that $\| D_{\eta^{\star}}^{-1/2} X D_{\eta^{\star}}^{-1/2} - \sqrt{\eta^{\star}} \sqrt{\eta^{\star}}^T \|_2 \geq 1$, thus making $(\mathcal{P}_1)$ either undefined in its objective function (if the norm is equal to 1) or not satisfying one of the constraints. We thus need to ensure that for any fixed $\delta_n$ there exists at least one matrix $X$ such that $\| D_{\eta^{\star}}^{-1/2} X D_{\eta^{\star}}^{-1/2} - \sqrt{\eta^{\star}} \sqrt{\eta^{\star}}^T \|_2 < 1$ with the remaining constraints satisfied. To do so, we introduce the transition matrix $M^{\star} = (P^{\star} + \widehat{P^{\star}})/2$ with $\widehat{P^{\star}}$ the time-reversed transition matrix of $P^{\star}$ which is the transition matrix of the optimal policy for problem (\ref{eq:opt.function.asymptotic.lambda}). Whereas $P^{\star}$ is not necessarily reversible w.r.t.~$\eta^{\star}$, it is the case for $M^{\star}$, thus yielding $\textrm{SLEM}(M^{\star}) < 1$. We also define $X = D_{\eta^{\star}} M^{\star}$. By construction of $X$, we have $X \geq 0$, $X = X^T$, $\sum_{i,j} X_{ij} = 1$, $X_{ij} = 0 ~ \textrm{if} ~ Q_{ij} = 0$ and $\sum_{j} X_{ij} \geq \wu{\eta}$. Furthermore, we have $\sum_{i}\big(\sum_{j} X_{ij} - \eta_i^{\star}\big)^2 = \sum_{i}\big(\sum_{j} \eta_i^{\star}M^{\star}_{ij} - \eta_i^{\star}\big)^2 =  \sum_{i}(\eta_i^{\star})^2\big(\sum_{j} M^{\star}_{ij} - 1 \big)^2 = 0$ which means that all the constraints are verified. In addition, since $M^{\star}$ is reversible w.r.t.~$\eta^{\star}$, we have $\| D_{\eta^{\star}}^{-1/2} X D_{\eta^{\star}}^{-1/2} - \sqrt{\eta^{\star}} \sqrt{\eta^{\star}}^T \|_2 = \| D_{\eta^{\star}}^{1/2} M^{\star} D_{\eta^{\star}}^{-1/2} - \sqrt{\eta^{\star}} \sqrt{\eta^{\star}}^T \|_2 = \textrm{SLEM}(M^{\star}) < 1$. This proves that $(\mathcal{P}_1)$ is well-defined.
\end{proof}

Solving the convex optimization problem $(\mathcal{P}_1)$ yields an optimal matrix $X_1 \in \mathbb{R}^{S\times S}$, from which we easily obtain the associated stationary distribution $\eta_1$ as well as the transition matrix of the associated Markov chain $P_1$ 
\begin{align*}
\eta_1(s) = \sum_{s' \in \mathcal{S}} X_1(s,s') \quad \textrm{and} \quad P_1(s,s') = \frac{X_1(s,s')}{\sum_{s'}X_1(s,s')}.
\end{align*}

\textbf{Step 2.} The distribution $\eta_1$ is stationary w.r.t.~the Markov chain $P_1$ (which is expected to have better mixing properties than $P^{\star}$), but it may not be feasible w.r.t.~the MDP dynamics. As a result, we must now find a stationary policy $\pi$ whose stationary distribution is closest to $\eta_1$. This is closely linked to the steady-state control problem from \cite{Akshay}, where it is proved that for an ergodic MDP the problem of finding a stationary policy given a target stationary state distribution is effectively decidable in polynomial time. If the steady-state control problem admits a solution, such a policy can be computed by simply solving a polynomial-size linear program. More precisely, we seek a policy $\pi$ in the set of non-negative reals $\{\pi_{s,a} | s \in \mathcal{S}, a \in \mathcal{A}_s \}$ such that
\begin{align*}
\forall s \in \mathcal{S}, \quad \displaystyle \sum_{s' \in \mathcal{S}, a \in \mathcal{A}_{s'}} \eta_1(s')p(s|s',a)\pi_{s',a} = \eta_1(s) \quad \textrm{and} \quad \displaystyle \sum_{a \in \mathcal{A}_s}\pi_{s,a} = 1.
\end{align*}
If the steady-state control problem does not admit a solution, we seek a stationary policy whose stationary state distribution is closest to $\eta_1$ w.r.t.~the $\ell^2$-norm by solving the following convex optimization problem $(\mathcal{P}_2)$ in $\pi$
\begin{align}
& \underset{\pi}{\text{minimize}} \quad
\displaystyle \sum_{s} \Big(\eta_1(s) - \displaystyle \sum_{s' \in \mathcal{S}, a \in \mathcal{A}_{s'}} \eta_1(s')p(s|s',a)\pi_{s',a}  \Big)^2 \tag{$\mathcal{P}_2$} \\
& \text{subject to \quad $\forall s \in \mathcal{S}, \quad \pi_{s,a} \geq 0 \quad \forall a \in \mathcal{A}_s \quad \textrm{and} \quad \displaystyle \sum_{a \in \mathcal{A}_s}\pi_{s,a} = 1$}. \nonumber
\end{align}
Since we do not know in advance if the steady-state control problem admits a solution or not, we directly solve problem $(\mathcal{P}_2)$ which encompasses both cases (its optimal value is 0 if the steady-state control problem admits a solution). This yields a policy denoted $\pi^{\star}_{\fmh}$.

%

\subsection{A SDP formulation of \fmh (\fmhsdp)}\label{app:ssection:SDP_variant}

We notice that step 1 of \fmh is by far the most computationally demanding, due to the complexity of the objective function and constrained set of problem $(\mathcal{P}_1)$. Fortunately, the symmetry constraint on $X$ leads to the symmetry of the matrix $D_{\eta^{\star}}^{-1/2} X D_{\eta^{\star}}^{-1/2} - \sqrt{\eta^{\star}} \sqrt{\eta^{\star}}^T$, which is a very useful property because it becomes easy to compute a subgradient of its spectral norm w.r.t.~$X$ (see e.g., \citet[Sect.~5.1,][]{boyd2004fastest}). We can thus apply subgradient descent to solve $(\mathcal{P}_1)$. However a projection on the constrained set is required at each step. We thus propose an alternative method to solve problem $(\mathcal{P}_1)$ that is
projection-free and hence more computationally efficient. Since this approach uses semidefinite programming, the resulting heuristic is called \fmhsdp.

The key observation is that the regularizing term in $(\mathcal{P}_1)$ partially ``takes into account'' the non-regularized one through the last constraint $\|\eta - \eta^{\star}\| \leq \delta_n$. Furthermore, the regularizing term corresponds (up to composition of a non-decreasing function) to minimizing the spectral norm of a symmetric matrix.  Drawing inspiration from \citet[Sect.~2.3,][]{boyd2004fastest}, we can express it as a semidefinite program (SDP) which can be solved efficiently using standard SDP solvers. Introducing a scalar variable $s$ to bound the spectral norm, step 1 of \fmh is replaced by the following SDP problem whose variables are the matrix $X$ and the scalar $s$
\begin{equation}
\begin{aligned}
& \underset{X, s}{\text{minimize}}
& & s  \\
& \text{subject to}
& & -s I \preceq D_{\eta^{\star}}^{-1/2} X D_{\eta^{\star}}^{-1/2} - \sqrt{\eta^{\star}} \sqrt{\eta^{\star}}^T \preceq sI \\
&&& X \geq 0, \quad  X = X^T, \quad X_{ij} = 0 ~ \textrm{if} ~ J_{ij} = 0 \\
&&& \sum_{(i,j) \in \mathcal{S}^2} X_{ij} = 1, \quad \sum_{j \in \mathcal{S}} X_{ij} \geq \underline{\eta}, \quad |\big( \sum_{j \in \mathcal{S}} X_{ij}\big)_i - \eta^{\star}_i| \leq (\delta_n)_i.
\end{aligned}
\label{SDP_problem}
\end{equation}
\fmhsdp is not only more computationally efficient due to its SDP formulation but it also loses the dependency on the hyper-parameter $\rho_n$ as only $\delta_n$ remains.


\section{Proofs}\label{app:proofs}

We first recall the performance loss suffered by the continuous relaxation in the bandit case, where the frequency $T_{\pi,n}/n$ is replaced by an allocation $\lambda$ in the simplex. In order to keep the notation as consistent as possible, consider a stochastic bandit problem with $S$ arms, let $\Delta_n = \{ \eta_n\in [0,1]^S: \eta_n(s) = \frac{T_{n}(s)}{n} \}$ and $\Delta$ be the discrete and continuous simplex over $S$ arms, where $\eta_{n}(s)$ is the frequency associated to $T_n$ pulls. Since in this case a policy directly selects arms rather than actions, the objective functions $\mathcal{L}_n$ and $\mathcal{L}$ coincide and we can write
\begin{align*}
\mathcal{L}(\eta) = \frac{1}{S}\displaystyle\sum_{s} \dfrac{\sigma^2(s)}{\eta(s)},
\end{align*}
where $\eta$ may be either a discrete or a continuous allocation. We have the following.

\begin{proposition}\label{prop:bandit.perf.loss}
Let $\eta^\star_n = \arg\min_{\eta_n\in\Delta_n} \mathcal{L}(\eta_n)$ be the optimal discrete allocation. As computing $\eta^*_n$ is NP hard, a standard solution is to first compute $\eta^\star = \arg\min_{\eta\in\Delta} \mathcal{L}(\eta)$ and then round it to obtain $\wt\eta_n$. If $\wt\eta_n$ is computed using efficient apportionment techniques~\citep[Chapter 12,][]{pukelsheim2006optimal}, then for any budget $n > 2S$ we have
\begin{align*}
\mathcal{L}(\wt\eta_n) - \mathcal{L}(\eta_n^{\star}) \leq \frac{2}{n}\displaystyle\sum_{s} \frac{\sigma^2(s)}{\eta^\star(s)} = \frac{2S}{n} \mathcal{L}(\eta^\star).
\end{align*}
Furthermore, for any $n \geq 4/(S\eta_{\min}^2)$, where $\eta_{\min} = \min_s \eta^{\star}(s)$ we have
\begin{align*}
\mathcal{L}(\wt\eta_n) - \mathcal{L}(\eta_n^{\star}) \leq \dfrac{8 \sigma^2_{\max}}{\eta_{\min}^3 n^2}.
\end{align*}

\end{proposition}

\begin{proof}
Using efficient apportionment techniques for rounding we have~\citep[Lem.~12.8,][]{pukelsheim2006optimal} 
\begin{align*}
\displaystyle \min_{s} \dfrac{\wt \eta_n(s)}{\eta^{\star}(s)} \geq 1 - \dfrac{S}{n}, \quad \quad \textrm{i.e.,} \quad \forall s,~ \wt\eta_n(s) \geq \eta^{\star}(s)\Big(1-\dfrac{S}{n}\Big),
\end{align*}
which also implies the other direction as
\begin{align*}
\wt\eta_n(s) = 1-\sum_{s'\neq s} \wt\eta_n(s') \leq 1-\sum_{s'\neq s} \eta^\star(s') + \sum_{s'\neq s}\eta^{\star}(s')\dfrac{S}{n} \leq \eta^\star(s) + \frac{S}{n}.
\end{align*}
Then we can bound the performance loss of $\wt\eta_n$ as
\begin{align*}
\mathcal{L}(\wt\eta_n) - \mathcal{L}(\eta_n^{\star}) = \mathcal{L}(\wt\eta_n) - \mathcal{L}(\eta^{\star}) + \underbrace{\mathcal{L}(\eta^{\star}) - \mathcal{L}(\eta_n^{\star})}_{\leq 0}
\leq \dfrac{1}{S}\displaystyle\sum_{s} \sigma^2(s) \Big( \dfrac{1}{\wt\eta_n(s)} - \dfrac{1}{\eta^{\star}(s)}\Big).
\end{align*}
Under the assumption that $n > 2 S$, we can bound each of the summands as
\begin{align*}
\dfrac{1}{\wt\eta_n(s)} - \dfrac{1}{\eta^{\star}(s)} = \dfrac{\eta^{\star}(s) - \wt\eta_n(s)}{\wt\eta_n(s)\eta^{\star}(s)} \leq \dfrac{\eta^\star(s)S/n}{(\eta^\star(s))^2(1- S/n)} \leq \dfrac{2S}{\eta^\star(s) n},
\end{align*}
which proves the $O(1/n)$ upper bound. Recalling the definition of $\mathcal{L}(\eta^\star)$ we obtain the final statement 
\begin{align*}
\mathcal{L}(\wt\eta_n) - \mathcal{L}(\eta_n^{\star}) \leq 2\displaystyle\sum_{s} \dfrac{\sigma^2(s)}{\eta^\star(s) n} = \frac{2S}{n} \mathcal{L}(\eta^\star).
\end{align*}
An even faster rate can be obtained exploiting the smoothness of $\mathcal{L}$. Let $\overline{\Delta} = \{ \eta \in \Delta: \forall s, \eta(s) \geq \eta_{\min}/2\}$. Since $\wt\eta(s) \geq \eta^\star(s)(1-S/n)$, for any any $n > 2S$ we have $\wt\eta_n, \eta^{\star} \in \overline{\Delta}$, hence using the $\overline{C}$-smoothness of $\mathcal{L}$ on $\overline{\Delta}$ with $\overline{C} = \frac{2\sigma^2_{\max}}{S(\eta_{\min}/2)^3}$, we can write~\citep[Thm.~12.10][]{pukelsheim2006optimal}
\begin{align*}
\mathcal{L}(\wt\eta_n) - \mathcal{L}(\eta_n^{\star}) \leq \dfrac{ \overline{C}}{2} ||\wt\eta_n - \eta^{\star}||_2^2 \leq \dfrac{8 \sigma^2_{\max}}{\eta_{\min}^3 n^2},
\end{align*}
which corresponds to an asymptotic rate of $O(1/n^2)$. Since the multiplicative constants are larger than those of the $O(1/n)$ rate, the rate $O(1/n^2)$ effectively starts when $n$ is big enough. A rough bound on $n$ for the second bound to be effectively smaller than the first is obtained by upper-bounding $\mathcal{L}(\eta^\star) \leq \sigma^2_{\max}/\eta_{\min}$ as
\begin{align*}
\dfrac{8 \sigma^2_{\max}}{\eta_{\min}^3 n^2} \leq \dfrac{2\sigma^2_{\max} S}{\eta_{\min} n} \iff n \geq \dfrac{4}{S \eta_{\min}^2},
\end{align*}
which concludes the proof.

\end{proof}

\begin{proof}[Proof of Proposition~\ref{prop.concentration.mixing}]
The first statement is a direct application of the relationship between mixing and spectral gap. For any policy reversible and ergodic policy $\pi$, any starting state $s'$ and any state $s$, we have from \citet[Prop.~3,][]{diaconis1991geometric}
\begin{align*}
\big| \mathbb{P}_\pi(s_t = s | s_1 = s') - \eta_\pi(s)\big| \leq \dfrac{1}{2} \sqrt{\dfrac{1-\eta_{\pi}(s')}{\eta_{\pi}(s')}} (1-\gamma_\pi)^t
\end{align*}
Then the difference between the expected frequency and the stationary distribution is bounded as
\begin{align*}
\Big| \frac{\mathbb{E}\big[T_{\pi,n}(s)\big]}{n} - \eta_\pi(s) \Big| &\leq \frac{1}{n}\sum_{t=2}^n \big| \mathbb{P}_\pi(s_t = s | s_1 = \overline{s}) - \eta_\pi(s)\big| \nonumber\\
&\leq \frac{1}{2\sqrt{\eta_{\min}} n} \sum_{t=1}^n (1-\gamma_\pi)^t \leq \frac{1}{2\sqrt{\eta_{\min}} n\gamma_\pi}.
\end{align*}

\end{proof}


\begin{proof}[Proof of Proposition~\ref{prop:asymptotic.relaxation}]
The state-action polytope $\Lambda$ is closed, bounded and convex according to \citet[Thm.~8.9.4,][]{puterman1994markov}. The problem \eqref{eq:opt.function.asymptotic.lambda} is thus convex in $\lambda$ due to the convexity of the objective function and constraints. It is straightforward that $\pi_{\lambda^{\star}} \in \Pi^{\textrm{SR}}$. From \citet[Thm.~8.8.1,][]{puterman1994markov}, the stationary distribution $\eta_{\pi_{\lambda^{\star}}}$ of $\pi_{\lambda^{\star}}$ is the unique solution of the system of equations $\sum_{s'} P_{\pi_{\lambda^{\star}}}(s|s') \eta_{\pi_{\lambda^{\star}}}(s') = \eta_{\pi_{\lambda^{\star}}}(s)$ (for each state $s$) subject to $\sum_{s} \eta_{\pi_{\lambda^{\star}}}(s)=1$. Given that $\big(\sum_{a} \lambda^{\star}(s,a)\big)_s$ is a solution, it corresponds to the stationary distribution $\eta_{\pi_{\lambda^{\star}}}$. By contradiction, assume that there exists a policy $\overline{\pi} \in \Pi^{\textrm{SR}}$ such that $\mathcal{L}(\overline{\pi}, \eta_{\overline{\pi}}) < \mathcal{L}(\pi_{\lambda^{\star}}, \eta_{\pi_{\lambda^{\star}}})$. Then define for every state-action pair $(s,a)$ the quantity $\overline{\lambda}(s,a) = \eta_{\overline{\pi}}(s) \overline{\pi}(a|s)$. It is evident that $\overline{\lambda} \in \Delta(\mathcal{S} \times \mathcal{A})$, furthermore for every state $s$, we have
\begin{align*}
    \displaystyle\sum_{s',a}p(s|s',a)\overline{\lambda}(s',a)  &= \displaystyle\sum_{s',a}p(s|s',a)\eta_{\overline{\pi}}(s) \overline{\pi}(a|s) = \sum_{s'} \eta_{\overline{\pi}}(s') \sum_{a} p(s|s',a) \overline{\pi}(a|s') \\ &= \sum_{s'} \eta_{\overline{\pi}}(s') P_{\overline{\pi}}(s|s') = \eta_{\overline{\pi}}(s) = \displaystyle\sum_{a} \overline{\lambda}(s,a),
\end{align*}
since by stationarity of the policy $\overline{\pi}$, the Markov chain transition matrix $P_{\overline{\pi}}$ is stationary w.r.t. $\eta_{\overline{\pi}}$. So $\overline{\lambda}$ satisfies the constraint of stationarity of \eqref{eq:opt.function.asymptotic.lambda}, and
\begin{align*}
    \mathcal{L}(\overline{\lambda}) = \sum_{s} \dfrac{\sigma^2(s)}{\eta_{\overline{\pi}}(s)} = \mathcal{L}(\overline{\pi}, \eta_{\overline{\pi}}) < \mathcal{L}(\pi_{\lambda^{\star}}, \eta_{\pi_{\lambda^{\star}}}) = \sum_{s} \dfrac{\sigma^2(s)}{\eta_{\pi_{\lambda^{\star}}}(s)} = \sum_{s} \dfrac{\sigma^2(s)}{\sum_{a} \lambda^{\star}(s,a)} = \mathcal{L}(\lambda^{\star}),
\end{align*}
which contradicts the optimality of $\lambda^{\star}$ for problem \eqref{eq:opt.function.asymptotic.lambda} and thus proves that $\pi_{\lambda^{\star}}$ is the optimal solution of the problem \eqref{eq:asymptotic.opt.policy}. Finally, the upper bound on the smoothness parameter $C_{\wu\eta}$ on the restricted set $\Lambda_{\wu{\eta}}$ is derived using that the maximal eigenvalue of a symmetric block matrix with positive eigenvalues is bounded by above by the sum of maximal eigenvalues of its diagonal blocks.
\end{proof}


\begin{proof}[Proof of Lemma~\ref{lem:asmyptoric.performance}] The proof is a rather direct application of Prop.~\ref{prop.concentration.mixing}. We first recall the exact formulation of the term $\epsilon_\pi(s,n,\delta)$ in Prop.~\ref{prop.concentration.mixing} (see e.g., \citet[Thm.~3,][]{Hsu2015mixing},~\citet[Thm.~3.8,][]{Paulin}):
\begin{align*}
\epsilon_\pi(s,n,\delta):= \sqrt{8 \eta_\pi(s)(1-\eta_\pi(s)) \dfrac{\ln(\frac{1}{\delta}\sqrt{\frac{2}{\eta_{\pi,\min}}})}{\gamma_\pi n}} + 20 \dfrac{\ln(\frac{1}{\delta}\sqrt{\frac{2}{\eta_{\pi,\min}}})}{\gamma_\pi n}.
\end{align*}
Let $\eta_{\pi,n}(s) = \frac{T_{\pi,n}(s)}{n}$ be the empirical frequency of visits to state $s$. Since we need all following statements to hold simultaneously for all states $s\in\calS$ and all stationary policies $\pi\in\Pi^{\textrm{SR}}$, we need to take a union bound over states and a cover over the action simplex at each state, which leads to tuning $\delta = \delta'/(SA^S)$ in the high-probability guarantees of Prop.~\ref{prop.concentration.mixing}, which then hold with probability $1-\delta'$. Furthermore, we have the following deterministic bound
\begin{align*}
\Big|\frac{1}{\eta_{\pi,n}(s)} - \frac{1}{\eta_{\pi}(s)}\Big| \leq \max\{n, \frac{1}{\eta_{\pi}(s)} \},
\end{align*}
where we used Asm.~\ref{asm:counter} to ensure that $1/\eta_{\pi,n}(s) \leq n$.
We introduce the event
\begin{align*}
\mathcal{E}_1(s,n,\delta) = \{\eta_{\pi,n}(s) \geq \eta_\pi(s) - \epsilon_\pi(s,n,\delta)\}.
\end{align*}
Then we have
\begin{align*}
\bigg| \mathbb{E}\Big[\frac{1}{\eta_{\pi,n}(s)} - \frac{1}{\eta_{\pi}(s)}\Big] \bigg| &\leq \bigg| \mathbb{E}\Big[\Big(\frac{1}{\eta_{\pi,n}(s)} - \frac{1}{\eta_{\pi}(s)}\Big)\mathbb{I}\{\mathcal{E}_1(s,n,\delta)\}\Big] \bigg| +
\bigg| \mathbb{E}\Big[\Big(\frac{1}{\eta_{\pi,n}(s)} - \frac{1}{\eta_{\pi}(s)}\Big)\mathbb{I}\{\mathcal{E}^{\mathtt{C}}_1(s,n,\delta)\}\Big] \bigg| \\
& \leq \bigg| \mathbb{E}\Big[\Big(\frac{1}{\eta_{\pi,n}(s)} - \frac{1}{\eta_{\pi}(s)}\Big)\mathbb{I}\{\mathcal{E}_1(s,n,\delta)\}\Big] \bigg| +
\max\{n, \frac{1}{\eta_{\pi}(s)}\}\mathbb{P}\{\mathcal{E}^{\mathtt{C}}_1(s,n,\delta)\}  \\
&\leq\Big|\mathbb{E}\Big[\frac{\eta_{\pi}(s) - \eta_{\pi,n}(s)}{\eta_{\pi,n}(s)\eta_{\pi}(s)} \mathbb{I}\{\mathcal{E}_1(s,n,\delta)\}\Big]\Big| + \max\{n, \frac{1}{\eta_{\pi}(s)}\}\delta'\\
&\leq \frac{\big|\mathbb{E}\big[\eta_{\pi}(s) - \eta_{\pi,n}(s)\big]\big|}{\eta_{\pi}(s)\big(\eta_{\pi}(s) - \epsilon_\pi(s,n,\delta)\big)} + \max\{n, \frac{1}{\eta_{\pi}(s)}\}\delta'\\
&\leq \frac{1}{2 \sqrt{\eta_{\min}}n\gamma_\pi \eta_{\pi}^2(s)}\Big( 1 + 2 \frac{\epsilon_\pi(s,n,\delta)}{\eta_\pi(s)}\Big) + \max\{n, \frac{1}{\eta_{\pi}(s)}\}\delta',
\end{align*}
where the last inequality follows from $1/(1-x) \leq 1+2x$ for $0 < x \leq 1/2$ which can be applied due to the condition that $n$ is big enough so that $\epsilon_\pi(s,n,\delta) \leq \eta_\pi(s)/2$. Since this condition requires $n \geq O(1/\eta_{\min}^2)$, we can resolve the maximum in the previous expression as $\max\{n, \frac{1}{\eta_{\pi}(s)}\} \leq n$. Finally, setting $\delta' = 1/n^2$ translates to the inequality on the objective function 
\begin{align*}
\big| \mathcal{L}_n(\pi) - \mathcal{L}(\pi, \eta_\pi) \big| \leq \ell_n(\pi) := \frac{1}{S \sqrt{\eta_{\min}}n\gamma_\pi}\sum_{s\in\calS} \frac{\sigma^2(s)}{\eta_\pi^2(s)}\Big( 1 + 2 \frac{\epsilon_\pi(s,n,\delta)}{\eta_\pi(s)}\Big),
\end{align*}
from which we obtain the final statement as
\begin{align*}
\mathcal{L}_n(\pi_{\lambda^\star}) - \mathcal{L}_n(\pi^\star_n) &\leq \mathcal{L}(\pi_{\lambda^\star}, \eta_{\pi_{\lambda^\star}}) + \ell_n(\pi_{\lambda^\star}) - \mathcal{L}(\pi^\star_n, \eta_{\pi^\star_n}) + \ell_n(\pi^\star_n) \\
&\leq \ell_n(\pi_{\lambda^\star}) + \ell_n(\pi^\star_n).
\end{align*}
\end{proof}


\begin{proof}[Proof of Lemma~\ref{lem:reg.performance}]
The proof relies on the concentration inequality in Eq.~\ref{eq:concentration.loss}. We proceed through the following inequalities
\begin{align*}
\mathcal{L}_n(\pi^\star_\reg) \leq \mathcal{L}(\pi^\star_\reg, \eta_{\pi^\star_\reg}) + \ell_n(\pi^\star_\reg) \leq \mathcal{L}(\pi^\star_n, \eta_{\pi^\star_n}) + \ell_n(\pi^\star_n) \leq \mathcal{L}_n(\pi^\star_n) + 2\ell_n(\pi^\star_n),
\end{align*}
where in the first and last inequality we used Eq.~\ref{eq:concentration.loss}, and where the second inequality follows from the definition of $\pi^\star_\reg$ as the optimal solution to the regularized problem.
\end{proof}


\begin{proof}[Proof of Lemma~\ref{lem:performance_loss_fmh}]
Introducing the term $H := \mathcal{L}^{\textrm{reg}}(\pi^{\star}_{\fmh}) - \mathcal{L}^{\textrm{reg}}(\pi^{\star}_{\textrm{reg}})$ where $\mathcal{L}^\textrm{reg}$ is defined in Eq.~\ref{eq:opt.function.regularized}, we have
\begin{align*}
	\mathcal{L}_n(\pi^{\star}_{\fmh}) \leq \mathcal{L}^{\textrm{reg}}(\pi^{\star}_{\fmh}) = H + \mathcal{L}^{\textrm{reg}}(\pi^{\star}_{\textrm{reg}}) \leq H + \mathcal{L}_n(\pi^\star_n) + 2\ell_n(\pi^\star_n).
\end{align*}

Given the expression of $\ell_n(\pi)$ provided in Lem.~\ref{lem:asmyptoric.performance}, we can write 
\begin{align*}
\mathcal{L}^{\text{reg}}(\pi,\eta) = \mathcal{L}(\pi,\eta) + \dfrac{\rho_n}{1-\| D_{\eta}^{1/2} P_\pi D_{\eta}^{-1/2} \!-\! \sqrt{\eta} \sqrt{\eta}^\transp \|_2} + O(n^{-3/2}).
\end{align*}
For notational simplicity we denote $\eta_{\textrm{fmh}} = \eta_{\pi^{\star}_{\fmh}}$, $P_{\textrm{fmh}} = P_{\pi^{\star}_{\fmh}}$, $\eta_{\textrm{reg}} = \eta_{\pi^{\star}_{\textrm{reg}}}$ and $P_{\textrm{reg}} = P_{\pi^{\star}_{\textrm{reg}}}$. We thus have
	\begin{align*}
	H = \sum_{s \in \mathcal{S}} \Big(\dfrac{\sigma^2(s)}{\eta_{\textrm{fmh}}(s)} - \dfrac{\sigma^2(s)}{\eta_{\textrm{reg}}(s)} \Big) + \rho_n \Big(\dfrac{1}{\gamma(P_{\textrm{fmh}})} - \dfrac{1}{\gamma(P_{\textrm{reg}})}\Big) + O(n^{-3/2}).
	\end{align*}
	Given that $\eta_{\textrm{reg}}$ is a stationary state distribution w.r.t. the MDP dynamics, we can write by optimality of $\eta^{\star}$ for the problem (\ref{eq:opt.function.asymptotic.lambda})
	\begin{align*}
	\sum_{s \in \mathcal{S}} \Big(\dfrac{\sigma^2(s)}{\eta_{\textrm{fmh}}(s)} - \dfrac{\sigma^2(s)}{\eta_{\textrm{reg}}(s)} \Big) \leq \sum_{s \in \mathcal{S}} \Big(\dfrac{\sigma^2(s)}{\eta_{\textrm{fmh}}(s)} - \dfrac{\sigma^2(s)}{\eta^{\star}(s)} \Big) \leq \dfrac{\sigma_{\max}^2}{\wu{\eta}^2} \|\eta_{\textrm{fmh}} - \eta^{\star}\|_1 \leq \dfrac{\sigma_{\max}^2 \sqrt{S}}{\wu{\eta}^2} \|\eta_{\textrm{fmh}} - \eta^{\star}\|_2.
	\end{align*}
	Using successively the triangular inequality, the property guaranteed in $(\mathcal{P}_2)$ that $\eta_{\textrm{fmh}}$ minimizes the distance $\|\cdot - \eta_1\|_2$ among all the stationary state distributions w.r.t.~the MDP dynamics, and finally the property guaranteed in $(\mathcal{P}_1)$ of $\delta_n$-proximity of $\eta_1$ to $\eta^{\star}$, we get
	\begin{align*}
	\|\eta_{\textrm{fmh}} - \eta^{\star}\|_2 \leq \|\eta_{\textrm{fmh}} - \eta_1\|_2 + \|\eta_1 - \eta^{\star}\|_2 \leq 2 \|\eta_1 - \eta^{\star}\|_2 \leq 2 \delta_n.
	\end{align*}
	We conclude the proof using the fact that $\gamma(P_{\textrm{fmh}})$ and $\gamma(P_{\textrm{reg}})$ are larger than $\gamma_{\min}$.
\end{proof}



\section{Proof of Thm.~\ref{thm:regret_FWAME}}\label{app:proof_regret}

\subsection{Preliminaries}

We recall that the notation $u_n = \wt O(v_n)$ means that there exist $c > 0$ and $d > 0$ such that $u_n \leq c (\log n)^d v_n$ for sufficiently large $n$. By abuse of language we say that a stationary policy $\pi$ belongs to $\Lambda_{\wu{\eta}}$ if $\forall s \in \cal{S}, \eta_{\pi} \geq \textrm{2} \wu{\eta}$. For notational convenience we consider throughout the proof that we relax Asm.~\ref{asm:counter} (cf.~App.~\ref{app:relaxing:asm:counter}) and that the initial state $s_1$ is drawn from an arbitrary initial distribution over states and we collect its observation $x_1$. This leads to the configuration where at every time $t$ exactly $t$ state samples have been collected. We start our analysis with the two following technical lemmas. 

\begin{lemma} \label{technical_lemma_1}
Let $\delta \in (0,1)$. For any length $\tau > 0$, the following bound holds simultaneously for any state $s$ and any policy $\pi \in \Lambda_{\wu{\eta}}$ with probability at least $1-\delta$
\begin{align*}
     \abs[\Big]{\dfrac{\sum_{t=1}^{\tau} \mathbb{I}\{\pi_t=s\}}{\tau} - \eta_{\pi}(s)} \leq M(\tau, \delta) := \sqrt{ \dfrac{2B}{\gamma_{\min}\tau}} + \dfrac{20B}{\gamma_{\min} \tau} \quad \quad \textrm{with} \quad B = \log\Big(\frac{S A^S}{\delta}\sqrt{\frac{1}{\wu{\eta}}}\Big).
\end{align*}
\end{lemma}

\begin{proof}
Pick any $\delta \in (0,1)$. Let $\pi$ be a fixed policy whose stationary distribution is lower-bounded by $\eta_{\pi,\min}$ and whose associated Markov chain admits $\gamma_{\pi}$ as spectral gap. For any length $\tau > 0$ and state $s$, we define $\nu_{\pi, \tau}(s) = \sum_{t=1}^{\tau} \mathbb{I}\{\pi_t=s\}$. From Prop.~\ref{prop.concentration.mixing}, for a fixed state $s \in \mathcal{S}$, the following bound holds with probability at least $1-\delta$
\begin{align*}
    \abs[\big]{\dfrac{\nu_{\pi, \tau}(s)}{\tau} - \eta_{\pi}(s)} &\leq \sqrt{8 \eta_{\pi}(s)(1-\eta_{\pi}(s)) \Tilde{\epsilon}} + 20 \Tilde{\epsilon} \quad \textrm{where} ~ \Tilde{\epsilon} = \dfrac{\log(\frac{1}{\delta}\sqrt{\frac{2}{\eta_{\pi,\min}}})}{\gamma_{\pi} \tau}.
\end{align*}
Since we need this statement to hold simultaneously for all states $s\in\mathcal{S}$ and all stationary policies $\pi\in \Lambda_{\wu{\eta}}$, we need to take a union bound over states and a cover over the action simplex at each state, which leads to tuning $\delta = \delta'/SA^S$ and thus yields with probability at least $1-\delta$
\begin{align*}
    \abs{\dfrac{\nu_{\pi, \tau}(s)}{\tau} - \eta_{\pi}(s)} &\leq \sqrt{8 \eta_{\pi}(s)(1-\eta_{\pi}(s)) \Tilde{\epsilon}} + 20 \Tilde{\epsilon} \quad \textrm{where} ~ \Tilde{\epsilon} = \dfrac{\log(\frac{SA^S}{\delta}\sqrt{\frac{1}{\wu{\eta}}})}{\gamma_{\min} \tau}.
\end{align*}
Using the fact that the function $x \mapsto x(1-x)$ is upper bounded by 1/4 and setting $B = \log(\frac{S A^S}{\delta}\sqrt{\frac{1}{\wu{\eta}}})$ yields the desired high-probability result.
\end{proof}

\begin{lemma} \label{technical_lemma_2}
Let $\delta \in (0,1)$. There exists a length $\tau_{\delta} > 0$ such that for any $T \geq \tau_{\delta}$, the following inequality holds simultaneously for any state $s$ and any policy $\pi \in \Lambda_{\wu{\eta}}$ with probability at least $1-\delta$
\begin{align*}
    \sum_{t=1}^T \mathbb{I}\{\pi_t=s\} \geq \wu{\eta} T
\end{align*}
\end{lemma}

\begin{proof}
Pick any $\delta \in (0,1)$. $M(\tau, \delta)$ is a decreasing function of $\tau$, hence there exists a length $\tau_{\delta}$ such that for any $T \geq \tau_{\delta}$, $M(T, \delta) \leq \wu{\eta}$. As a result, Lem.~\ref{technical_lemma_1} guarantees that we have with probability at least $1-\delta$ simultaneously for any state $s$ and any stationary policy $\pi \in \Lambda_{\wu{\eta}}$ 
\begin{align*}
     \abs{\dfrac{\sum_{t=1}^{T} \mathbb{I}\{\pi_t=s\}}{T} - \eta_{\pi}(s)} \leq \wu{\eta},
\end{align*}
which yields in particular
\begin{align*}
    \dfrac{\sum_{t=1}^{T} \mathbb{I}\{\pi_t=s\}}{T} \geq \eta_{\pi}(s) - \wu{\eta} \geq \wu{\eta}.
\end{align*}

\end{proof}

Restricting our attention to increasing episode lengths in \fwmdp and using Lem.~\ref{technical_lemma_2}, we deduce the important property that for any $\delta \in (0,1)$, there exists an episode $k_{\delta}$ such that for all episodes $k$ succeeding it (and including it), we have with probability at least $1-\delta$ 
\begin{align}
    \sum_{a \in \mathcal{A}}\wt{\lambda}_{k}(s,a) \geq \wu{\eta}, \quad \forall s \in \mathcal{S}, \quad \forall k \geq k_{\delta}.
    \label{high_proba_lambda_tilde}
\end{align}
More specifically, $k_{\delta}$ is the first episode whose length $\tau_{k_{\delta}}$ verifies
\begin{align}
     M(\tau_{k_{\delta}}, \delta) =\sqrt{ \dfrac{2B}{\gamma_{\min} \tau_{k_{\delta}}}} + \dfrac{20B}{\gamma_{\min} \tau_{k_{\delta}}} \leq \wu{\eta} \quad \quad \textrm{with} \quad B = \log\Big(\frac{S A^S}{\delta}\sqrt{\frac{1}{\wu{\eta}}}\Big).
\label{k_delta}
\end{align}

We proceed by providing time-dependent lower and upper bounds on the true gradient $\nabla \mathcal{L}$, which is unknown. We denote by $\wh{\mathcal{L}}_{t}^{+}$ the empirical optimistic approximation of $\mathcal{L}$ at any time $t$, i.e.,
\begin{align*}
    \wh{\mathcal{L}}_t^{+}(\lambda) = \sum_{s \in \mathcal{S}} \dfrac{1}{\sum_a \lambda(s,a)}\Big[\wh{\sigma}_t^2(s) + 5R^2\sqrt{\dfrac{\log(\frac{4St}{\delta})}{T_{t}(s)}}\Big] = \wh{\mathcal{L}}_t(\lambda) + \sum_{s \in \mathcal{S}} \dfrac{\alpha(t,s,\delta)}{\sum_a \lambda(s,a)}.
\end{align*}
Here we used that $\nu(s)$ is an observation distribution supported in $[0,R]$. We note that this assumption can be easily extended to the general case of sub-Gaussian distributions as done in~\cite{carpentier2011upper}. From Prop.~\ref{lem:bound_variance_estimates}, the following inequalities hold with probability at least $1-\delta$ for any $\lambda$, time $t$ and state-action pair $(s,a)$
\begin{align}\label{ineq_squeeze}
    \nabla \wh{\mathcal{L}}_{t}^{+}(\lambda)(s,a) = \nabla \wh{\mathcal{L}}_{t}(\lambda)(s,a) - \dfrac{\alpha(t,s,\delta)}{(\sum_b \lambda(s,b))^2} \leq \nabla \mathcal{L}(\lambda)(s,a) \leq \nabla \wh{\mathcal{L}}_{t}(\lambda)(s,a) + \dfrac{\alpha(t,s,\delta)}{(\sum_b \lambda(s,b))^2}.
\end{align}

Finally, let $T=t_{K}-1$ be the final budget (i.e.,~the time at the end of the final episode $K-1$). For the sake of clarity and readability, we make the simplification that the logarithmic term $\log(T)$ behaves as a constant.

\subsection{Core of the proof} \label{core_proof}

We denote by $\rho_{k+1}$ the approximation error at the end of each episode $k$ (i.e., at time $t_{k+1}-1)$. Recalling that $\beta_k = \tau_{k}/(t_{k+1}-1)$, we have
\begin{align*}
    \rho_{k+1} = \mathcal{L}(\wt{\lambda}_{k+1}) - \mathcal{L}(\lambda^{\star}) = \mathcal{L}\big((1-\beta_k) \wt{\lambda}_k + \beta_k \wt{\psi}_{k+1}\big) - \mathcal{L}(\lambda^{\star}).
\end{align*}
Let $\psi^{\star}_{k+1} = \textrm{argmin}_{\lambda \in \Lambda_{\underline{\eta}}} \langle \nabla \mathcal{L}(\wt{\lambda}_k), \lambda \rangle$ be the state-action stationary distribution that ``exact'' \fw would return at episode $k$. We have the following series of inequality
\begin{align}
    \rho_{k+1} &\leq \mathcal{L}(\wt{\lambda}_k) - \mathcal{L}(\lambda^{\star}) + \beta_k \langle \nabla \mathcal{L}(\wt{\lambda}_k), \wt{\psi}_{k+1} - \wt{\lambda}_k \rangle + C_{\underline{\eta}} \beta_k^2 \nonumber \\
    &= \mathcal{L}(\wt{\lambda}_k) - \mathcal{L}(\lambda^{\star}) + \beta_k \langle \nabla \mathcal{L}(\wt{\lambda}_k), \psi^{\star}_{k+1} - \wt{\lambda}_k \rangle + C_{\underline{\eta}} \beta_k^2 + \beta_k \langle \nabla \mathcal{L}(\wt{\lambda}_k), \wt{\psi}_{k+1} - \psi^{\star}_{k+1} \rangle \nonumber\\
    &\leq \mathcal{L}(\wt{\lambda}_k) - \mathcal{L}(\lambda^{\star}) + \beta_k \langle \nabla \mathcal{L}(\wt{\lambda}_k), \lambda^{\star} - \wt{\lambda}_k \rangle + C_{\underline{\eta}} \beta_k^2 + \beta_k \langle \nabla \mathcal{L}(\wt{\lambda}_k), \wt{\psi}_{k+1} - \psi^{\star}_{k+1} \rangle \nonumber\\
    &\leq (1-\beta_k) \rho_k + C_{\underline{\eta}} \beta_k^2 + \beta_k \underbrace{\langle \nabla \mathcal{L}(\wt{\lambda}_k), \wh{\psi}^{+}_{k+1} - \psi^{\star}_{k+1} \rangle}_{\epsilon_{k+1}} + \beta_k \underbrace{\langle \nabla \mathcal{L}(\wt{\lambda}_k), \wt{\psi}_{k+1} - \wh{\psi}^{+}_{k+1} \rangle}_{\Delta_{k+1}},
\label{recurrence}
\end{align}
where the first step follows from the $C_{\underline{\eta}}$-smoothness of $\mathcal{L}$, the second inequality comes from the \fw optimization step and the definition of $\psi^{\star}_{k+1}$, which gives $\langle \nabla \mathcal{L}(\wt{\lambda}_k), \psi^{\star}_{k+1} - \wt{\lambda}_k \rangle \leq \langle \nabla \mathcal{L}(\wt{\lambda}_k), \lambda^{\star} - \wt{\lambda}_k \rangle$, the final step follows from the convexity of $\mathcal{L}$. The term $\epsilon_{k+1}$ measures the error due to an inaccurate estimate of the gradient and the term $\Delta_{k+1}$ refers to the discrepancy between the stationary state-action distribution $\wh{\psi}^{+}_{k+1}$ and the empirical frequency $\wt{\psi}_{k+1}$ of its realization for $\tau_{k}$ steps. 

\noindent \textbf{Step 1 (Bound on error $\Delta_{k+1}$).} For any $k \geq k_{\delta}$, inequality \eqref{high_proba_lambda_tilde} is verified and we can write
\begin{align*}
    \langle \nabla \mathcal{L}(\wt{\lambda}_k), \wt{\psi}_{k+1} - \wh{\psi}^{+}_{k+1} \rangle &= \displaystyle \sum_{s} \dfrac{-\sigma^2(s)}{\big(\sum_b \wt{\lambda}_k(s,b)\big)^2} \displaystyle \sum_{a} \big( \wt{\psi}_{k+1}(s,a) - \wh{\psi}^{+}_{k+1}(s,a) \big) \leq \dfrac{S \sigma^2_{\textrm{max}}}{\underline{\eta}^2}\|\dfrac{\nu_{k+1}}{\tau_k} - \eta_{\wh{\pi}^{+}_{k+1}}\|_{\infty}.
\end{align*}
Let $B = \log(\frac{S A^S}{\delta}\sqrt{\frac{1}{\underline{\eta}}})$. From Lem.~\ref{technical_lemma_1}, we have with probability at least $1-\delta$ simultaneously for every state $s$ and every policy followed during the episode
\begin{align*}
    \abs[\big]{\dfrac{\nu_{k+1}(s)}{\tau_{k}} - \eta_{\wh{\pi}^{+}_{k+1}}(s)} \leq \sqrt{ \dfrac{2B}{\gamma_{\min}\tau_{k}}} + \dfrac{20B}{\gamma_{\min} \tau_{k}}.
\end{align*}
Hence we obtain the following bound on $\Delta_{k+1}$ with probability at least $1-\delta$
\begin{align*}
\Delta_{k+1} &\leq  \dfrac{S \sigma^2_{\textrm{max}}}{\underline{\eta}^2} \Big[\sqrt{ \dfrac{2B}{\gamma_{\min}\tau_{k}}} + \dfrac{20B}{\gamma_{\min} \tau_{k}}\Big].
\end{align*}

\noindent \textbf{Step 2 (Bound on error $\epsilon_{k+1}$).} Using inequality \eqref{ineq_squeeze}, we get with probability at least $1-\delta$
\begin{align*}
    \langle \nabla \mathcal{L}(\wt{\lambda}_k), \wh{\psi}^{+}_{k+1} \rangle &= \sum_{s,a}\wh{\psi}^{+}_{k+1}(s,a) \nabla \mathcal{L}(\wt{\lambda}_k)(s,a) \\
    &\leq \sum_{s,a}\wh{\psi}^{+}_{k+1}(s,a) \nabla \wh{\mathcal{L}}_{t_k - 1}(\wt{\lambda}_k)(s,a) + \sum_{s,a}\wh{\psi}^{+}_{k+1}(s,a)\dfrac{\alpha(t_k - 1,s,\delta)}{(\sum_b \wt{\lambda}_k(s,b))^2} \\
    &\leq \sum_{s,a}\wh{\psi}^{+}_{k+1}(s,a) \nabla \wh{\mathcal{L}}_{t_k - 1}^{+}(\wt{\lambda}_k)(s,a) + 2\sum_{s,a}\wh{\psi}^{+}_{k+1}(s,a)\dfrac{\alpha(t_k - 1,s,\delta)}{(\sum_b \wt{\lambda}_k(s,b))^2} \\
    &\leq \langle \nabla \wh{\mathcal{L}}_{t_k - 1}^{+}(\wt{\lambda}_k), \psi^{\star}_{k+1} \rangle + 2\sum_{s,a}\wh{\psi}^{+}_{k+1}(s,a)\dfrac{\alpha(t_k - 1,s,\delta)}{(\sum_b \wt{\lambda}_k(s,b))^2} \\
    &\leq \langle \nabla \mathcal{L}(\wt{\lambda}_k), \psi^{\star}_{k+1} \rangle + 2\sum_{s,a}\wh{\psi}^{+}_{k+1}(s,a)\dfrac{\alpha(t_k - 1,s,\delta)}{(\sum_b \wt{\lambda}_k(s,b))^2}.
\end{align*}
For notational simplicity we denote by $T_k(s) = T_{t_k -1}(s)$ the number of visits of state $s$ until the end of episode $k-1$ (i.e., at time $t_k - 1$). Using inequality (\ref{high_proba_lambda_tilde}) and an intersection bound over two high-probability events, we get with probability at least $1-2\delta$ for any episode $k \geq k_{\delta}$
\begin{align*}
    \epsilon_{k+1} &\leq \sum_{s,a}\wh{\psi}^{+}_{k+1}(s,a) \dfrac{10 R^2}{\underline{\eta}^2} \sqrt{\log(\frac{4S(t_{k}-1)}{\delta})} \dfrac{1}{\sqrt{T_k(s)}} \\
    &\leq c_0 \underbrace{ \sum_{s,a}\wt{\psi}_{k+1}(s,a) \dfrac{1}{\sqrt{T_k(s)}}}_{v_{k}} + \underbrace{c_0 \sum_{s,a} \big(\wh{\psi}^{+}_{k+1}(s,a) - \wt{\psi}_{k+1}(s,a)\big) \dfrac{1}{\sqrt{T_k(s)}}}_{\xi_{k+1}},
\end{align*}
where we define $c_0 = \dfrac{10 R^2}{\underline{\eta}^2} \sqrt{\log(\frac{4ST}{\delta})}$. $\xi_{k+1}$ can be bounded in the same vein as $\Delta_{k+1}$ using Lem.~\ref{technical_lemma_1}. The error $\xi_{k+1}$ is of a higher order than $\Delta_{k+1}$ and for proof simplicity we consider the following loose bound which is satisfied with probability at least $1-\delta$
\begin{align*}
    \xi_{k+1} &\leq c_0 \Big[\sqrt{ \dfrac{2B}{\gamma_{\min}\tau_{k}}} + \dfrac{20B}{\gamma_{\min} \tau_{k}}\Big].
\end{align*}

\noindent \textbf{Step 3 (putting everything together in~\eqref{recurrence}).} For $k \geq k_{\delta}$, we get with probability at least $1-2\delta$
\begin{align*}
    \Delta_{k+1} + \xi_{k+1} &\leq \dfrac{c_1}{\sqrt{\tau_{k}}} + \dfrac{c_2}{\tau_{k}} \quad \textrm{with} \quad
    \begin{cases}
    ~ c_1 = \Big(c_0 + \dfrac{S \sigma^2_{\textrm{max}}}{\underline{\eta}^2}\Big) \sqrt{ \dfrac{2B}{\gamma_{\min}}} \\
    ~ c_2 = \Big(c_0 + \dfrac{S \sigma^2_{\textrm{max}}}{\underline{\eta}^2}\Big)\dfrac{20B}{\gamma_{\min}}
    \end{cases},
\end{align*}
which provides the bound for $k \geq k_{\delta}$
\begin{align}
    \rho_{k+1} \leq (1-\beta_k)\rho_k + \beta_k \big(  \dfrac{c_1}{\sqrt{\tau_{k}}} + \dfrac{c_2}{\tau_{k}}\big) + C_{\underline{\eta}}\beta_k^2 + \beta_k c_0 v_{k}.
\label{new_recurrence}
\end{align}
Choosing episode lengths satisfying $t_k = \tau_1 (k-1)^3 + 1$ yields
\begin{align*}
    \tau_k = t_{k+1} - t_{k} = \tau_1(3k^2-3k+1) \geq 3 \tau_1 k^2 \quad \textrm{and} \quad \beta_k =  \dfrac{\tau_{k}}{t_{k+1}-1} = \dfrac{3k^2 - 3k+1}{k^3} \in \Big[\dfrac{1}{k}, \dfrac{3}{k}\Big].
\end{align*}
Consequently we get
\begin{align}
    \beta_k \big( \dfrac{c_1}{\sqrt{\tau_{k}}} + \dfrac{c_2}{\tau_{k}}) + C_{\underline{\eta}}\beta_k^2 \leq \dfrac{b_{\delta}}{k^2} \quad \textrm{with} \quad b_{\delta} = \dfrac{\sqrt{3}c_1}{\sqrt{\tau_1}} + \dfrac{c_2}{\tau_1 k_{\delta}} + 9 C_{\underline{\eta}}.
    \label{ineq_b_delta}
\end{align}
Hence the recurrence inequality (\ref{new_recurrence}) becomes
\begin{align}
    \rho_{k+1} \leq (1-\dfrac{1}{k})\rho_k + \dfrac{b_{\delta}}{k^2} + \beta_k c_0 v_{k}.
\label{new_new_recurrence}
\end{align}

\noindent We pick an integer $q \geq (S/\tau_1)^{1/3} + 1$ such that $\rho_q \geq 0$ is satisfied.\footnote{Assuming this last condition is sensible since as the number of samples increases, $\wt{\lambda}$ gets closer to the stationary set $\Lambda_{\underline{\eta}}$ whose minimizer of $\mathcal{L}$ is $\lambda^{\star}$. The introduction of the term $ (S/\tau_1)^{1/3} + 1$ is motivated by the subsequent analysis of the series $\sum v_k$ in Lem.~\ref{lem:seq_v}.} We define the sequence $(u_n)_{n \geq q}$ as $u_q = \rho_q$ and
\begin{align*}
    u_{n+1} = \big(1-\dfrac{1}{n}\big)u_n + \dfrac{b_{\delta}}{n^2} + \beta_n c_0 v_{n},
\end{align*}
with $b_{\delta}$ the fixed positive constant defined in (\ref{ineq_b_delta}). From inequality (\ref{new_new_recurrence}), we have $\rho_k \leq u_k$ for $k \geq k_{\delta}$ and an immediate induction guarantees the positivity of the sequence $(u_n)$. By rearranging we get
\begin{align*}
    (n+1)u_{n+1} - nu_n = \dfrac{-u_n}{n} + \dfrac{b_{\delta}(n+1)}{n^2} + (n+1)\beta_n c_0 v_{n} \leq \dfrac{b_{\delta}(n+1)}{n^2} + (n+1)\beta_n c_0 v_{n}.
\end{align*}
By telescoping and using the fact that $\beta_n \leq 3/n \leq 6/(n+1)$, we obtain
\begin{align*}
    n u_n - q u_q \leq 2b_{\delta} \sum_{i=q}^{n-1} \dfrac{1}{i} + 6 c_0\sum_{i=q}^{n-1} v_i  \leq 2b_{\delta}\log(\frac{n-1}{q-1}) + 6 c_0 \sum_{i=q}^{n-1} v_i.
\end{align*}
Let $K \geq k_{\delta}$. We thus have with probability at least $1-2\delta$
\begin{align}
    \rho_K \leq \dfrac{q \rho_q + 2b_{\delta}\log K}{K} + \dfrac{6c_0}{K}\sum_{k=q}^{K-1} v_k =  \dfrac{\tau_1^{1/3}}{(t_K-1)^{1/3} + \tau_1^{1/3}}\Big( q \rho_q + 2b_{\delta}\log K + 6 c_0 \sum_{k=q}^{K-1} v_k \Big).
\label{final_bound}
\end{align}

We conclude the proof by plugging the result of Lem.~\ref{lem:seq_v} into inequality \eqref{final_bound} which yields the desired high-probability bound $\rho_K = \wt O(1/t_K^{1/3})$.

\begin{lemma} \label{lem:seq_v} Recalling that $v_k = \displaystyle \sum_{s,a}\wt{\psi}_{k+1}(s,a) \dfrac{1}{\sqrt{T_k(s)}}$, we have $\sum v_k = \wt O(1)$.
\end{lemma}

\begin{proof}
Denoting $\mathcal{S} = \{1, 2, ..., S\}$ and recalling that $q \geq (S/\tau_1)^{1/3} +1$, we have
\begin{align*}
    \sum_{k=q}^{K-1} v_{k} &= \sum_{k=q}^{K-1}  \sum_{s,a} \wt{\psi}_{k+1}(s,a) \dfrac{1}{\sqrt{T_k(s)}}
    = \sum_{k=q}^{K-1} \sum_{s=1}^S \dfrac{\sqrt{\nu_{k+1}(s)}}{\tau_{k}} \dfrac{\sqrt{\nu_{k+1}(s)}}{\sqrt{T_k(s)}} \\
    &\leq \sqrt{\sum_{k=q}^{K-1} \sum_{s=1}^S \dfrac{\nu_{k+1}(s)}{\tau_{k}^2}} \sqrt{\sum_{k=q}^{K-1} \sum_{s=1}^S \dfrac{\nu_{k+1}(s)}{T_k(s)}}
    = \sqrt{\underbrace{\sum_{k=q}^{K-1}  \dfrac{1}{\tau_{k}}}_{\Sigma_1}} \sqrt{\underbrace{\sum_{k=q}^{K-1} \sum_{s=1}^S \Big( \dfrac{T_{k+1}(s)}{T_k(s)}-1\Big)}_{\Sigma_2}},
\end{align*}
where the inequality uses the Cauchy-Schwarz inequality on the sum indexed doubly by the episodes and the states. Since the Riemann zeta function of 2 is upper bounded by 3, we have
\begin{align*}
    \Sigma_1 \leq \dfrac{1}{3\tau_1}\sum_{k=q}^{K-1} \dfrac{1}{k^2} \leq \dfrac{1}{\tau_1}.
\end{align*}
There remains to show that $\Sigma_2 = \wt O(1)$. We introduce the following related optimization problem. For any $K \geq q$, we have $t_{K} - 1 \geq S$ since we chose $q \geq (S/\tau_1)^{1/3} +1$. Let $V^{\star}(K)$ be defined by
\begin{align}
    V^{\star}(K) ~ = ~ &\max~ \sum_{k=q}^{K-1} \sum_{s=1}^S \big( h_{s,k} - 1 \big), \label{related_objective} \\
    & \textrm{s.t.} \quad h_{s,k} \geq 1 \quad \textrm{and} \quad \sum_{s=1}^S \prod_{k=q}^{K-1} h_{s,k} \leq t_{K} - 1. \label{related_constraints}
\end{align}
We have for any episode $k$ and state $s$, $T_{k+1}(s) \geq T_k(s)$ and
\begin{align*}
    \sum_{s=1}^S \prod_{k=q}^{K-1} \dfrac{T_{k+1}(s)}{T_k(s)} = \sum_{s=1}^S \dfrac{T_{K}(s)}{T_q(s)} \leq t_{K} - 1.
\end{align*}
Hence the sequence $\Big(\dfrac{T_{k+1}(s)}{T_k(s)}\Big)_{s,k}$ satisfies the constraints (\ref{related_constraints}), thus $\Sigma_2 \leq V^{\star}(K)$. There remains to solve the optimization problem (\ref{related_objective}). Since the variables $h_{s,k}$ play interchangeable roles, there exists $h^{\star} = h_{s,k}$ for all $s$ and $k$. From the second constraint in (\ref{related_constraints}), we know that $h^{\star} \leq \big((t_{K} -1) / S\big)^{1/(K-q)}$. Given that (\ref{related_objective}) is a maximization problem that increases proportionally with $h^{\star}$, when $t_{K}-1 \geq S$ (so as to satisfy the first constraint), we finally have $h^{\star} = \big((t_{K}-1) / S\big)^{1/(K-q)}$. Consequently we have
\begin{align*}
    \Sigma_2 &\leq \sum_{k=q}^{K-1} \sum_{s=1}^S \Big( \Big(\dfrac{t_{K}-1}{S}\Big)^{1/(K-q)} - 1 \Big) = \underbrace{\dfrac{\exp \Big(\dfrac{1}{K-q}\log\big(\dfrac{\tau_1 (K-1)^3}{S}\big)\Big) - 1}{\dfrac{1}{K-q} \log\big(\dfrac{\tau_1 (K-1)^3}{S}\big)}}_{\longrightarrow 1 \quad \textrm{when} \quad K \longrightarrow +\infty} \underbrace{S \log\big(\dfrac{\tau_1 (K-1)^3}{S}\big)}_{= \wt{\mathcal{O}}(1)},
\end{align*}
which proves that $\Sigma_2 = \wt O(1)$. We conclude the proof using that $\sum_{k=q}^{K-1} v_{k+1} \leq \dfrac{1}{\sqrt{\tau_1}} \sqrt{\Sigma_2}$.

\end{proof}

%

\begin{minipage}{0.47\textwidth}

\begin{algorithm}[H]
	\caption{\fwmdp w/ \fmhsdp}
	\label{alg:FW-AME-FMH-SDP}
	\begin{algorithmic} 
		\STATE \textbf{Input:} $\wt{\lambda}_1 = 1/SA, \wu{\eta}$\\
		\FOR{$k=1, 2, ..., K-1$}
		\STATE $\wh{\psi}^{+}_{k+1} = \textrm{argmin}_{\lambda \in \Lambda_{\wu{\eta}}} \langle \nabla \wh{\mathcal{L}}_{t_k-1}^{+}(\wt{\lambda}_k), \lambda \rangle$
		\STATE $\wh{\pi}^{+}_{k+1}(a|s) = \dfrac{\wh{\psi}^{+}_{k+1}(s,a)}{\sum_{b \in \cal{A}}\wh{\psi}^{+}_{k+1}(s,b)}$		
		\STATE \textcolor{blue}{Compute $\pi^{\star} _{\fmh} = \fmhsdp(\wh{\pi}^{+}_{k+1}, \tau_k)$ with $\delta_{\tau_k}$ defined in Eq.~(\ref{delta_tau_k})}
		\STATE Execute \textcolor{blue}{$\pi^{\star} _{\fmh}$} for $\tau_{k}$ steps, collect the samples and update $\wt{\lambda}_{k+1}$ as in Alg.~\ref{alg:FW-AME}
		\ENDFOR
	\end{algorithmic}
\end{algorithm}

\end{minipage} \hfill
\begin{minipage}{0.05\textwidth}
\end{minipage} \hfill
\begin{minipage}{0.48\textwidth}
\begin{figure}[H]
\includegraphics[width=\linewidth]{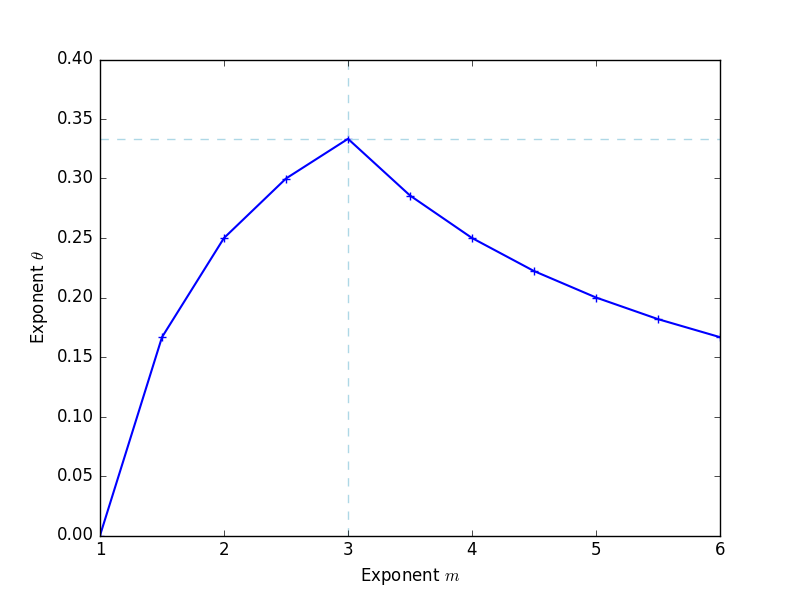}
\caption{Exponent $\theta$ as a function of $m$ (cf.~App.~\ref{ssec_bound_optimality}).}
\label{fig:theta_m}
\end{figure}
\end{minipage}

\subsection{Optimality of the Episode Length} \label{ssec_bound_optimality}

As explained in Sect.~\ref{ssec:learning}, an interesting open question is whether the regret bound obtained in Thm.~\ref{thm:regret_FWAME} is optimal. Our analysis however yields the following optimality result: among all the episode lengths such that the time $t$ is polynomial in the number of episodes $k$, i.e., among all the integers $m \geq 1$ such that $t$ behaves as $k^{m}$, the value of $m$ that optimizes convergence is $m=3$. Indeed, we can apply the Euler method on inequality \eqref{new_recurrence} which results in solving the differential equation $y' = \dfrac{-y}{x} + \dfrac{1}{x^2} + \dfrac{1}{x^{(m+1)/2}}$ and finding the largest $\theta$ such that $x^{\theta}y(x) = \wt O(1)$. $\theta$ is thus the largest value such that
    \begin{align*}
        \mathcal{L}(\wt{\lambda}_{k+1}) - \mathcal{L}(\lambda^{\star}) = \wt O \Big(\dfrac{1}{t^{\theta}}\Big) = \wt O\Big(\dfrac{1}{k^{\theta/m}}\Big).
    \end{align*}
Fig.~\ref{fig:theta_m} plots the exponent $\theta$ as a function of $m$ and shows that $\theta$ reaches its maximal value of $1/3$ for $m=3$, consequently yielding the regret $\wt O (1/t^{1/3})$.

\subsection{\fwmdp w/ \fmhsdp} \label{FWAME_FMHSDP}

The variant incorporating the framework of \fmhsdp is presented in Alg.~\ref{alg:FW-AME-FMH-SDP} and its difference with Alg.~\ref{alg:FW-AME} is highlighted in blue. 
The regret analysis is the same as in App.~\ref{core_proof} except that the error $\Delta_{k+1}$ in the recurrence inequality \eqref{recurrence} goes from $\wt{\psi}_{k+1} - \wh{\psi}^{+}_{k+1}$ to
\begin{align*}
	\wt{\psi}_{k+1}^{\fmh} - \wh{\psi}^{+}_{k+1} = \wt{\psi}_{k+1}^{\fmh} - \wh{\psi}^{\fmh}_{k+1} + \wh{\psi}^{\fmh}_{k+1} - \wh{\psi}^{+}_{k+1},
\end{align*}
where $\wt{\psi}_{k+1}^{\fmh} = \fmhsdp(\wh{\psi}^{+}_{k+1}, \tau_k)$ and $\wt{\psi}_{k+1}^{\fmh}$ is its empirical realization for the $\tau_k$ steps of the episode. The new error $\Delta_{k+1}$ can thus be decomposed as follows
\begin{align*}
    \underline{\eta}^2 \Delta_{k+1} \leq \sum_{s} \sigma^2(s) \abs[\big]{\dfrac{\nu^{\fmh}_{k+1}(s)}{\tau_k} - \eta_{\wh{\pi}^{\fmh}_{k+1}}(s)} + \sum_{s} \sigma^2(s) \abs[\big]{\eta_{\wh{\pi}^{\fmh}_{k+1}}(s) - \eta_{\wh{\pi}^{+}_{k+1}}(s)},
\end{align*}
where the first term is $O\Big(1/\sqrt{\gamma(\wh{\psi}^{\fmh}_{k+1}) \tau_k}\Big)$ and the second term is upper bounded by $\sum_s \sigma^2(s) \delta_{\tau_k}$ where $\delta_{\tau_k}$ is the \fmhsdp parameter from problem \eqref{SDP_problem}.

Since w/o \fmhsdp we have $\gamma_{k+1} = O\Big(1/\sqrt{\gamma(\wh{\psi}^{\fmh}_{k+1}) \tau_k}\Big)$, this suggests that the slack variable $\delta_{\tau_k}$ can decrease at least as $O(1/\sqrt{\tau_k})$ so as to guarantee that the order of the error $\Delta_{k+1}$ is unchanged. Furthermore, the component $\delta_{\tau_k}(s)$ is weighted by $\sigma^2(s)$ (which is unknown), hence we are encouraged to set
\begin{align} \label{delta_tau_k}
\delta_{\tau_k}(s) = \dfrac{\wh\Sigma - \wh{\sigma}_{t_k-1}^2(s)}{(S-1)\wh\Sigma} \dfrac{1}{\sqrt{\tau_k}} \quad \textrm{where} ~ \wh\Sigma = \sum_{s \in \cal{S}} \wh{\sigma}_{t_k-1}^2(s).
\end{align}

The regret analysis of \fwmdp w/ \fmhsdp is thus unchanged and we recover the final rate in $O(t^{-1/3})$. In addition, if the heuristic is able to obtain an improvement in the mixing properties of the episodic policy (i.e., $\gamma(\wh{\psi}^{\fmh}_{k+1})$ bigger than $\gamma(\wh{\psi}^{+}_{k+1})$) that outweighs the error introduced by $\delta_{\tau_k}(s)$, then the regret performance at episode $k$ of \fwmdp w/ \fmhsdp is improved.


\section{Garnet MDPs}\label{app:experiments}

We detail here the process for generating Garnet\footnote{In full, Generalized Average Reward Non-stationary Environment Test-bench \citep{bhatnagar2009natural}.} MDPs which we use in Sect.~\ref{sec:experiments}. A Garnet instance ${\cal G}(S, A, b, \sigma_{\min}^2, \sigma_{\max}^2)$ is characterized by 5 parameters. $S$ and $A$ are the number of states and actions respectively, and $b$ is a branching factor specifying the number of possible next states for each state-action pair, i.e., the number of uniformly distributed non-zero entries in each line of the MDP transition matrix. We ensure the aperiodicity of the MDP by adding a non-zero probability (equal to 0.001) of self-loop for all state-action pairs. Since the state means are arbitrarily fixed, there remains to uniformly sample the state variances $\sigma^2(s)$ between $\sigma^2_{\textrm{min}}$ and $\sigma^2_{\textrm{max}}$ and randomly select two states whose variances are set respectively to $\sigma^2_{\textrm{min}}$ and $\sigma^2_{\textrm{max}}$. We likewise introduce reversible Garnet MDPs denoted by $\cal{G}_{\cal R}$. The generation process of $\cal{G}_{\cal R}$ is identical to $\cal{G}$ except that we set the branching factor to $b-1$ and ensure the reversibility of the MDP by randomly picking $a \in \cal{A}$ and $q \in (0,1)$ such that $p(s|s',a) = q$ for every pair $(s,s')$ such that $Q(s,s')=1$ (and finally normalize to obtain an admissible $p$). 

We note that the Garnet procedure allows some control over the mixing properties of the MDP. Indeed, when $A$ and $b$ are small, only a few transitions are assigned significant probabilities so the speed of mixing is generally slower. For higher values of $A$ and $b$, all the positive transition probabilities are of similar magnitude so the speed of mixing is generally faster.

\end{appendices}

\end{document}